\documentclass[letterpaper, 11 pt]{article} 
\usepackage{amsmath,amsfonts,amssymb,color}
\usepackage{mathtools}
\usepackage{dsfont}
\usepackage[dvipsnames]{xcolor}
\definecolor{mygreen}{rgb}{0.0, 0.5, 0.0}
\definecolor{winered}{rgb}{0.8,0,0}
\definecolor{myblue}{rgb}{0,0,0.8}
\usepackage{tikz}
\usepackage{amsthm}
\usepackage{empheq}
\usepackage{wrapfig}

\newtheorem{definition}{Definition}
\newtheorem{theorem}{Theorem}
\newtheorem{lemma}{Lemma}
\newtheorem{proposition}{Proposition}

\newtheorem{remark}{Remark}

\DeclareMathOperator*{\argmax}{\arg\!\max}

\DeclarePairedDelimiterX{\norm}[1]{\lVert}{\rVert}{#1}

\usepackage{graphicx}
\usepackage{caption}
\usepackage{subcaption}
\usepackage{hyperref}
\usepackage{epsfig}
\usepackage{array}
\usepackage{multirow}
\usepackage{epstopdf}
\usepackage{tikz}
\usepackage{relsize}
\usetikzlibrary{shapes,arrows}
\hypersetup{
    colorlinks=true,
    linkcolor={winered},
    citecolor={mygreen}
}
\usepackage{subcaption}
\usepackage{cite}
\usepackage{enumitem}
\newcommand{\ef}{f_{\theta_p}}
\usepackage{xspace}
\usepackage{amsmath}
\usepackage{graphicx} 
\usepackage{amsfonts}

\usepackage{amsthm}
\newcommand\blfootnote[1]{%
  \begingroup
  \renewcommand\thefootnote{}\footnote{#1}%
  \addtocounter{footnote}{-1}%
  \endgroup
}

\usepackage{algorithm}
\usepackage{algpseudocode}%
\usepackage{caption}
\usepackage{subcaption}
\usepackage{wrapfig}
\newtheorem{Problem}{Problem}
\usepackage[margin=1in]{geometry}
\title{Robust Information Selection for Hypothesis Testing with Misclassification Penalties}
\author{Jayanth Bhargav, Shreyas Sundaram and Mahsa Ghasemi}
\date{}
\begin{document}
\maketitle
\begin{abstract}
We study the problem of robust information selection for a Bayesian hypothesis testing / classification task, where the goal is to identify the true state of the world from a finite set of hypotheses based on observations from the selected information sources. We introduce a novel misclassification penalty framework, which enables non-uniform treatment of different misclassification events. Extending the classical subset selection framework, we study the problem of selecting a subset of sources that minimize the maximum penalty of misclassification under a limited budget, despite deletions or failures of a subset of the selected sources. We characterize the curvature properties of the objective function and propose an efficient greedy algorithm with performance guarantees. Next, we highlight certain limitations of optimizing for the maximum penalty metric and propose a submodular surrogate metric to guide the selection of the information set. We propose a greedy algorithm with near-optimality guarantees for optimizing the surrogate metric. Finally, we empirically demonstrate the performance of our proposed algorithms in several instances of the information set selection problem. \blfootnote{This material is based upon work supported by the Office of Naval Research (ONR) via Contract No. N00014-23-C-1016 and under subcontract to Saab, Inc. as part of the TSUNOMI project. Any opinions, findings and conclusions or recommendations expressed in this material are those of the author(s) and do not necessarily reflect the views of ONR, the U.S. Government, or Saab, Inc. \\
The authors are with the Elmore Family School of Electrical and Computer Engineering, Purdue University, West Lafayette IN 47907 USA. Email addresses: \{\tt jbhargav, \tt sundara2, \tt mahsa\}\tt @purdue.edu}
\end{abstract}

\section{Introduction}
In many autonomous systems, agents rely on predictions made by classifiers to make decisions or take actions. However, these predictions can be compromised when critical information sources are disrupted due to adversarial attacks or sensor failures, leading to costly and potentially catastrophic outcomes. For example, in a surveillance system, an intruder could jam or manipulate sensors to evade detection, or natural failures such as sensor outages could cause the system to lose critical observations. Such disruptions not only degrade the quality of predictions, but also amplify the risk of misclassifications, particularly in high-stakes environments. To ensure reliable decision-making in these settings, it is crucial to design robust selection strategies that maintain system performance even in the presence of adversarial actions or failures.
In this paper, we study the problem of robust information selection for a hypothesis testing/classification task, which explicitly accounts for the possibility of disruptions to the selected information sources. We consider the setting where a system designer must select a subset of information sources, ensuring that the selected set can withstand adversarial attacks or unforeseen failures while still ensuring certain desired system performance, which is quantified by a misclassification penalty framework.
As a motivating example, consider a surveillance task where the goal is to accurately identify targets, such as distinguishing between a drone (an intruder) and a bird. Misclassifying a drone as a bird could result in a security breach, while misclassifying a bird as a drone might lead to unnecessary deployment of resources. These differing costs are captured using a penalty matrix, which quantifies the impact of different types of misclassification errors. Now imagine a scenario where an adversary disables certain sensors or where a sensor fails due to environmental factors. A robust information set selection ensures that even under these disruptions, the selected subset of information sources collectively provides sufficient reliability to ensure minimal misclassification penalty.

\subsection{Related Work}
Misclassification risk and uncertainty quantification for various types of classifiers have been very well studied in the literature \cite{adams1999comparing,pendharkar2017bayesian,hou2013modeling}. In \cite{sensoy2021misclassification}, the authors propose a risk-calibrated classifier to reduce the costs associated with misclassification errors, and empirically show the effectiveness of their algorithm in a  deep learning framework. In \cite{elkan2001foundations}, the authors study cost-sensitive learning for class balancing in order to improve the quality of predictions in decision tree learning methods. In our work, we consider a hypothesis testing (or classification) task in a Bayesian learning framework. 

A subset of the literature has addressed the problem of sequential information gathering within a limited budget \cite{hollinger2013sampling,chen2015sequential}. The authors of  \cite{golovin2010near} study data source selection for a monitoring application, where the sources are selected sequentially in order to estimate certain parameters of an environment. In \cite{ghasemi2019online}, the authors study sequential information gathering under a limited budget for a robotic navigation task. In \cite{gupta2006stochastic}, the authors study sequential sensor scheduling to jointly estimate a process and present a stochastic selection algorithm which is computationally tractable.  In contrast, we consider the scenario where the information set is selected \textit{a priori}, i.e., at \textit{design-time}, and propose an efficient algorithm with guarantees. 

A substantial body of work focuses on the study of submodularity (and/or weak submodularity) properties for efficient greedy techniques with provable guarantees for feature selection in sparse learning \cite{krause2010submodular,chepuri2014sparsity}, sensor selection for estimation \cite{mo2011sensor,hashemi2020randomized,shamaiah2010greedy,krause2007near} \& Kalman filtering \cite{ye2018complexity}, and observation selection for mixed-observable Markov decision processes \cite{bhargav2023complexity}.  Along the lines of these works, we leverage the weak submodularity property of the performance metric and present greedy algorithms with performance guarantees.

Robust sensor selection has been extensively studied in the context of resource-constrained environments where sensors may fail, be removed, or experience adversarial attacks \cite{ye2020complexity,laszka2015resilient,oh2023dynamic}. Early works, such as those by \cite{krause2008robust,kaya2025randomized,laszka2015resilient}, focused on optimizing submodular objectives to achieve near-optimal sensor placement under budget constraints, ensuring reliable performance despite uncertainty. In the robust setting, adversarial or stochastic failures were explored by \cite{tzoumas2018resilient}, where greedy algorithms were developed to guarantee near-optimal performance. The authors of \cite{kaya2025randomized} study robust weak-submodular optimization of a set of weak-submodular functions, where the goal is to maximize the utility of the worst-case objective. Along similar lines as these works, we study the robust information selection problem where the objective is not submodular, and present greedy algorithms with near-optimality guarantees. Furthermore, we present a submodular surrogate metric which enjoys improved performance guarantees for greedy approximation.
\newpage
\subsection{Contributions} 
First, we propose a novel misclassification penalty framework for hypothesis testing tasks that captures the varying importance of different classification errors, enabling non-uniform treatment of misclassification events. Extending the classical information selection paradigm, we formulate the problem of robust information subset selection under a limited budget to minimize the maximum penalty of misclassifying a hypothesis, accounting for worst-case disruptions. We characterize the curvature properties of the objective function and propose an efficient greedy algorithm for robust subset selection with near-optimality guarantees. Next, we highlight some limitations in the achievable performance guarantees and propose a submodular surrogate metric based on the total penalty of misclassification. We propose a greedy algorithm with near-optimality guarantees for robust subset selection that optimizes for the surrogate metric. Finally, we empirically validate our framework and algorithms in diverse scenarios, demonstrating their effectiveness in ensuring robust and cost-effective decision-making under uncertainty. 

Part of the results presented in this paper were previously presented in the conference paper \cite{bhargav24a}. Specifically, in \cite{bhargav24a}, we studied the optimal information subset selection problem for hypothesis testing and established the weak-submodularity property of the objective function by characterizing it's submodularity ratio. In this work, we consider the robust version of the information set selection problem and characterize various curvature properties of the objective function. Building on existing works in the literature on non-submodular optimization, we propose efficient greedy techniques with provable performance guarantees for approximating the solution. Additionally, we propose a submodular surrogate metric for information set selection. We characterize the curvature properties of the surrogate objective and propose an efficient greedy algorithm which enjoys improved performance bounds. 
\section{Background and Problem Formulation}
\label{sec:prob_form}
In this section, we present some preliminaries and formulate the robust information set selection problem.

Let $\Theta = \{ \theta_1, \theta_2, \hdots, \theta_m \}$, where $m = |\Theta| $, be a finite set of possible hypotheses (also referred to as classes or states), of which one of them is the true state of the world. We consider a set $\mathcal{D} = \{1,2, \hdots, n \}$ of information sources (or data streams) from which we need to select a subset $\mathcal{I} \subseteq \mathcal{D}$. At each time step $t \in \mathbb{Z}_{\geq 1}$, the observation provided by the information source $i \in \mathcal{D}$ is denoted as $o_{i,t} \in O_i$, where $O_i$ is the observation space of the source $i$. Each information source $i \in \mathcal{D}$ is associated with an observation likelihood function $\ell_i(\cdot | \theta)$, which is conditioned on the state of the world $\theta \in \Theta$. At time $t$, conditioned on the true state of the world $\theta \in \Theta$, a joint observation profile of $n$ information sources, denoted as $o_t = (o_{1,t}, \hdots, o_{n,t}) \in \mathcal{O}$ where $\mathcal{O} = O_1 \times \hdots \times O_n$, is generated by the joint likelihood function $\ell(\cdot|\theta)$. We make the following assumption on the observation model (e.g., see \cite{jadbabaie2012non,liu2014social,lalitha2014social} for detailed discussions regarding this assumption).\\

\textbf{Assumption 1:} \textit{The observation space $O_i$ associated with each information source $i \in \mathcal{D}$ is finite, and the likelihood function $\ell_i(\cdot | \theta)$ satisfies $\ell_i(\cdot | \theta) > 0$ for all $o_i \in O_i$ and for all $\theta \in \Theta$. We assume that the designer knows $\ell_i(\cdot | \theta)$ for all $\theta \in \Theta$ and all $i \in \mathcal{D}$. For all $\theta \in \Theta$, conditioned on the true state, the observations $o_i$ are independent of each other. For a source $i \in \mathcal{D}$, the observation sequence $\{ o_{i,1}, o_{i,2}, \hdots \}$ is a sequence of independent identically distributed (i.i.d.) random variables, given a true state $\theta \in \Theta$. } \\

\newpage
Consider the scenario where a system designer (also referred to as a central node) needs to select a subset of information sources in order to identify the true state of the world. For any subset $\mathcal{I} \subseteq \mathcal{D}$ with $|\mathcal{I}| = k$, let $\mathcal{I} = \{s_1, s_2, \hdots, s_k\}$ denote the information sources. The joint observation conditioned on the $\theta \in \Theta$ of this information set at time $t$ is defined as $o_{\mathcal{I},t} = \{ o_{s_1,t}, \hdots, o_{s_k,t} \} \in O_{s_1} \times \hdots \times O_{s_k} $, and is generated by the joint likelihood function $\ell_{\mathcal{I}}(\cdot | \theta) = \Pi_{i = 1}^{k} \ell_{s_i}(\cdot | \theta)$ (by Assumption 1). The designer knows $\ell_{\mathcal{I}} (\cdot | \theta)$ for all $\mathcal{I} \subseteq \mathcal{D}$ and for all $\theta \in \Theta$. For a true state $\theta_p \in \Theta$, we define $\mathbb{P}^{\theta_p} = \prod_{t=1}^{\infty}\ell(\cdot|\theta_p)$ to be the probability measure. For the sake of brevity, we will say that an event occurs almost surely to mean that it occurs almost surely w.r.t. the probability measure $\mathbb{P}^{\theta_p}.$ 
As the data comes in, the central node updates its belief over the set of possible hypotheses using the standard Bayes' rule. Let $\mu_{t}^{\mathcal{I}}(\theta)$ denote the belief of the central designer (or node) that $\theta$ is the true hypothesis  at time step $t$ based on the information sources in $\mathcal{I}$, and let $\mu_0(\theta)$ denote the initial belief (or prior) of the central node that $\theta$ is the true state of the world, with $\sum_{\theta \in \Theta} \mu_0(\theta) = 1$. The Bayesian update rule is given by 
\begin{equation}
    \label{eq:bayes_full}
    \mu_{t+1}^{\mathcal{I}}(\theta) = \frac{\mu_0(\theta)  \prod_{j=0}^{t} \ell_{\mathcal{I}} (o_{\mathcal{I}, j+1} | \theta) }{\sum_{\theta_i \in \Theta} \mu_0(\theta_i) 
  \prod_{j=0}^{t} \ell_{\mathcal{I}} (o_{\mathcal{I}, j+1} | \theta_i) } \quad \forall \theta \in \Theta.
\end{equation}

For a hypothesis $\theta \in \Theta$ and an information set $\mathcal{I} \subseteq \mathcal{D}$, we have the following.

\begin{definition}[Observationally Equivalent Set \cite{ye2021near}] For a given hypothesis (or class) $\theta \in \Theta$ and a given $\mathcal{I} \subseteq \mathcal{D}$, the \textit{observationally equivalent set} of classes to $\theta$ is defined as
\begin{equation}
    F_{\theta}(\mathcal{I}) = \{ \theta_i \in \Theta \mid D_{KL} (\ell_\mathcal{I}(\cdot | \theta_i) || \ell_\mathcal{I}(\cdot | \theta) ) = 0 \},
\end{equation}
where $D_{KL}(\ell_\mathcal{I}(\cdot | \theta_i) || \ell_\mathcal{I}(\cdot | \theta) )$ is the Kullback-Leibler divergence measure between $\ell_\mathcal{I}(\cdot | \theta_i)$ and $ \ell_\mathcal{I}(\cdot | \theta) $.
\end{definition}
From the definition above, we have $\theta \in F_{\theta}(\mathcal{I})$ for all $\theta \in \Theta$ and for all $\mathcal{I} \subseteq \mathcal{D}$. We can write the set $F_{\theta}(\mathcal{I})$ as
\begin{equation}
\label{eq:obs_eq_set}
    F_{\theta}(\mathcal{I}) = \{ \theta_i \in \Theta : \ell_{\mathcal{I}}(o_{\mathcal{I}} | \theta_i) = \ell_{\mathcal{I}}(o_{\mathcal{I}} | \theta), \forall o_{\mathcal{I}} \in \mathcal{O}_{\mathcal{I}}   \},
\end{equation}
where $\mathcal{O}_{\mathcal{I}}  = O_{s_1} \times \hdots \times O_{s_k}$ is the joint observation space of the information set $\mathcal{I}$. In other words, $F_{\theta}(\mathcal{I})$ is the set of hypotheses (or classes) that cannot be distinguished from $\theta$ based on the observations obtained by the information sources in $\mathcal{I}$. Furthermore, by Assumption 1 and Equation (\ref{eq:obs_eq_set}), we have the following (see Section 2 in \cite{ye2021near}):
\begin{equation}
    \label{eq: obs_eq_intersect}
    F_{\theta}(\mathcal{I}) = \bigcap_{s_i \in \mathcal{I}} F_{\theta} (s_i), \forall \mathcal{I} \in \mathcal{D}, \forall \theta \in \Theta.
\end{equation}
We define $F_{\theta}(\emptyset) = \Theta$, i.e.,  when there is no information set, all classes are observationally equivalent.
 At time $t$, the central designer predicts the state of the world based on the belief  $\mu_{t}^{\mathcal{I}}$ generated by the information set $\mathcal{I}$. In order to characterize the learning performance, we consider a penalty-based classification framework. Let  $\Xi = [\xi_{ij}] \in \mathbb{R}^{m \times m}$ denote the \textit{penalty matrix}, where $ 0 \leq \xi_{ij} \leq 1$ is the penalty associated with predicting the class to be $\theta_j$, given that the true class  is $\theta_i$. The penalty matrix is assumed to be \textit{row stochastic}, i.e., $\sum_{j = 1}^{m} \xi_{ij} = 1$. We have $\xi_{ii} = 0, \hspace{4pt} \forall i \in \{1,2, \hdots, m\}$, i.e., there is no penalty when the predicted hypothesis is the true hypothesis.  Similar to analyses presented in \cite{nedic2017fast,mitra2020new}, we characterize the asymptotic convergence of the Bayesian belief over the set of hypotheses. We consider the case of a uniform prior, but the results can be extended to non-uniform priors (using similar arguments as in Lemma 1 of \cite{mitra2020new}).
\newpage
\begin{theorem}
\label{thm2}
    Let the true state of the world be $\theta_p$ and let $\mu_0 (\theta) = \frac{1}{m} \hspace{5pt} \forall \theta \in \Theta$ (i.e., uniform prior). Under Assumption 1, for an information set $\mathcal{I} \subseteq \mathcal{D}$, the Bayesian update rule in Equation (\ref{eq:bayes_full}) has the following property: 
    \begin{enumerate}[label=(\alph*)]
        \item  $ \displaystyle \lim_{t \to \infty} \mu_t^{\mathcal{I}}(\theta_q) =  \mu_t^{\mathcal{I}}(\theta_p) \hspace{4pt} a.s.,  \hspace{5pt} \forall \theta_q \in F_{\theta_p}(\mathcal{I})$, and
        \item  $\displaystyle \lim_{t \to \infty} \mu_t^{\mathcal{I}}(\theta_q) = 0 \hspace{4pt} a.s.,\hspace{5pt} \forall \theta_q \notin F_{\theta_p}(\mathcal{I})$.
    \end{enumerate}
\end{theorem}
\begin{proof}
    See proof of Theorem 2 in \cite{bhargav24a}.
\hfill   \end{proof}
Theorem \ref{thm2} implies the following: if one has a uniform prior over the set of hypotheses, then the Bayesian update rule ensures that the asymptotic belief over the classes which are observationally equivalent to the true class have the same belief as that of the true class and the belief over the classes which are not observationally equivalent to the true class approach zero. This means that the central node can predict any $\theta \in F_{\theta_p}(\mathcal{I})$ to be the true class with the same probability. 

Consider the scenario in which the central designer has a fixed budget $K$ to select information sources and seeks to minimize the maximum penalty of misclassifying the true state. However, an attacker may remove up to $A < K$ information sources, or alternatively, up to $A$ sources may fail, making their observations unavailable. Since the true state is not known a priori, the central designer has to minimize the maximum penalty for each possible true state, which is a multi-objective optimization problem under a budget constraint. We scalarize the multi-objective optimization into a single-objective optimization problem. The optimal solution to this single-objective problem is a Pareto optimal solution to the multi-objective problem \cite{hwang2012multiple}. The goal of the designer is to select a subset of up to $K$ sources $\mathcal{I} \subseteq  \mathcal{D}, |\mathcal{I}| \le K$, such that the worst-case removal of up to $A$ sources ($\mathcal{I}' \subseteq \mathcal{I}, |\mathcal{I}'| \le A$) still ensures a minimal misclassification penalty. We now formalize the Robust Minimum Penalty Information Set Selection (R-MPIS) Problem.

\begin{Problem}[R-MPIS]
\label{prob:mpis}
    Consider a set $\Theta=\left\{\theta_1, \ldots, \theta_m\right\}$ of possible states of the world; a set $\mathcal{D}$ of information sources; a row-stochastic penalty matrix $ \Xi = [\xi_{ij}] \in \mathbb{R}^{m \times m}$ ; a selection budget $K \in \mathbb{R}_{\geq0}$; and an attack budget $A \in \mathbb{R}_{\geq0}$. The R-MPIS Problem is to find a set of selected information sources $\mathcal{I} \subseteq \mathcal{D}$ that solves
\begin{equation}
\label{eq:mpis}
\begin{aligned}
 \min _{\mathcal{I} \subseteq \mathcal{D}} \max_{\mathcal{I}' \subseteq \mathcal{I}} & \hspace{5pt} \displaystyle \sum_{\theta_p \in \Theta} \left(\max_{\theta_j \in F_{\theta_p}(\mathcal{I} \setminus \mathcal{I}')} \xi_{pj} \right) ;\\
  \hspace{2pt} \text { s.t. } & |\mathcal{I}|  \leq K, |\mathcal{I}'|\leq A.
\end{aligned}
\end{equation}
\end{Problem}
Consider the following optimization problem: 
\begin{equation}
\label{eq:mpis_max}
\begin{aligned}
 \max _{\mathcal{I} \subseteq \mathcal{D}} \min _{\mathcal{I} \subseteq \mathcal{I}'} & \hspace{5pt} \displaystyle \sum_{\theta_p \in \Theta} \left(1-\max_{\theta_j \in F_{\theta_p}(\mathcal{I} \setminus \mathcal{I}')} \xi_{pj} \right);
  \hspace{2pt}\\ \textit { s.t. } & |\mathcal{I}|  \leq K, |\mathcal{I}'|\leq A.
\end{aligned}
\end{equation}
 We note that a min-max optimization problem can be equivalently transformed into a max-min optimization problem by negating the objective function. Additionally, introducing a constant to the objective function of an optimization problem does not change its optimizer. Thus, the problem defined in (\ref{eq:mpis_max}) is equivalent to the problem defined in (\ref{eq:mpis}), i.e., the information sets $\mathcal{I}, \mathcal{I}' \subseteq \mathcal{D}$ that optimize the problem in Equation (\ref{eq:mpis_max}) are also optimal for Equation \eqref{eq:mpis}. We will now refer to the optimization problem in (\ref{eq:mpis_max}) as the R-MPIS problem. We denote $f_{\theta_p}(\mathcal{I}) = 1-\max_{\theta_j \in F_{\theta_p}(\mathcal{I})} \xi_{pj}$, and the objective function as $\Lambda(\mathcal{I}\setminus \mathcal{I}') = \sum_{\theta_p \in \Theta} f_{\theta_p}(\mathcal{I}\setminus \mathcal{I}')$. 

\section{Algorithm for Robust Information Selection}
The R-MPIS problem is a bilevel subset selection problem. Using similar arguments as in \cite{krause2008robust}, R-MPIS can be shown to be NP-hard. As a result, finding the exact solution can be computationally intractable. In this section, we analyze certain properties of the objective function of R-MPIS, which we will leverage to provide an efficient greedy algorithm for robust subset selection, with near-optimality guarantees. 

\subsection{Theoretical Characterization of the R-MPIS Objective}
We begin by defining certain properties and ratios of set functions, which were introduced in various works on submodular and non-submodular optimization \cite{vondrak1978submodularity,das2018approximate,bogunovic2018robust}. 

\begin{definition}[Monotonicity] A set function $f:2^{\Omega }\to\mathbb{R}$ is monotone non-decreasing if $f(X) \leq f(Y), \forall X \subseteq Y \subseteq \Omega$ and monotone non-increasing if $f(X) \geq f(Y), \forall X \subseteq Y \subseteq \Omega$.
    
\end{definition}

\begin{definition}[Submodularity Ratio\footnote{There are several notions of submodularity ratio. We consider $\gamma_{U,k}$ as defined in \cite{das2018approximate}, where $U$ is the universal set and $k\ge 1$ is a parameter, and drop the dependence on $k$ by defining $\gamma = \inf_{k} \gamma_{U,k}$. }]
\label{def:submod_ratio}
 Given a set  $\Omega$, the submodularity ratio of a non-negative function $f: 2^{\Omega} \to \mathbb{R}$ is the largest $\gamma \in \mathbb{R}$ that satisfies  for all $A, B \subseteq \Omega$, the following:
 \begin{equation*}
 \label{eq:submod_ratio_def}
     \sum_{a \in A \setminus B} (f(\{a\} \cup B) - f(B)) \ge \gamma (f(A \cup B) - f(B)).
 \end{equation*}
\end{definition}
\begin{remark}
\label{remark:submod_ratio}
    For a non-negative and non-decreasing function $f(\cdot)$ with submodularity ratio $\gamma$, we have $\gamma \in [0,1]$. If $\gamma$ is closer to 1, the function is closer to being submodular. $f(\cdot)$ is submodular if and only if $\gamma = 1$.  In order to characterize the performance of greedy approximation algorithms, one has to give a (nonzero) lower bound on $\gamma$ \cite{das2018approximate}.
\end{remark}

\begin{definition}[Inverse Generalized Curvature] The inverse generalized curvature of a monotone set function $f: 2^{\Omega} \to \mathbb{R}$ is the smallest scalar $\check{\alpha} \in [0,1]$ s.t.
\begin{equation}
\begin{aligned}
\label{eq:inverse_curvature}
		\frac{f(S) - f(S \setminus \{i\})}{f(S \cup V) - f((S\setminus \{i\}) \cup V)} \geq  1 - \check{\alpha},\\ \quad \forall S, V \subseteq \Omega, i \in S \setminus V.
\end{aligned}
\end{equation}
\end{definition}

\begin{definition}[Superadditivity Ratio]
    The superadditivity ratio of a monotone set function $f: 2^{\Omega} \to \mathbb{R}$ is the largest scalar $\check{\nu} \in [0,1]$ such that 
	\begin{equation} \label{eq:superadditivity}
		\frac{f(S)}{\sum_{i \in S} f(\lbrace i \rbrace)} \geq \check{\nu}, \quad \forall S \subseteq \Omega.
	\end{equation}
\end{definition}
\begin{definition}[Bipartite Subadditivity Ratio]\label{eq:two_separate_subbaditivity}
	The bipartite subadditivity ratio of $f: 2^{\Omega} \to \mathbb{R}$ is the largest scalar $\kappa \in [0,1]$ s.t.
	\begin{equation} \label{eq:2_separate_sa}
		\frac{f(A) + f(B)}{f(S)} \geq \kappa, \quad \forall S \subseteq \Omega, A \cup B = S, A \cap B = \emptyset.
	\end{equation}
\end{definition}

It follows from (\ref{eq: obs_eq_intersect}) that $f_{\theta_p}(\cdot)$ is a monotone non-decreasing set function with $f_{\theta_p}(\emptyset) = 1-\max_{\theta_j\in \Theta} \xi_{pj}$.

We now characterize the submodularity and curvature ratios of the function $f_{\theta_p}(\cdot)$. We have the following assumption on the misclassification penalties.\\ 

\noindent \textbf{Assumption 2:} The misclassification penalties are unique, i.e., $\xi_{pi} \neq \xi_{pj}$ for all $i \neq j, \forall \theta_p \in \Theta$.  \\
\noindent Note that the above assumption requires that no two misclassification events have the same penalty associated with them, which is often a reasonable assumption in many applications.
\begin{lemma} \label{lma:approx_submod} Under Assumption 2, the function $f_{\theta_p} :2^{\mathcal{D} }\to\mathbb{R}_{\geq0} $ is approximately submodular with a submodularity ratio $\gamma_p = \underline{\xi}_p/ \Bar{\xi}_p $, where 
\begin{align}
\label{eq:xifloor}
\displaystyle \underline{\xi}_p = \min_{\theta_i,\theta_j \in \Theta, \theta_i \ne \theta_j}  |\xi_{pi} - \xi_{pj}| ; \\
    \Bar{\xi}_p = \max_{\theta_i,\theta_j \in \Theta, \theta_i \ne \theta_j}  |\xi_{pi} - \xi_{pj}|.
\end{align}
\end{lemma}
\begin{proof}
Recall that $f_{\theta_p}(\mathcal{I}) =1- \max_{\theta_i\in F_{\theta_p}(\mathcal{I})} \xi_{pi}.$  We begin by proving the following statement: For $A, B \subseteq \mathcal{D}$,
\begin{equation}
\begin{aligned}
    \sum_{a \in A \setminus B} (\ef(\{a\} \cup B) - \ef(B)) = 0 \\
    \implies  \ef(A \cup B) - \ef(B) = 0.
    \end{aligned}
\end{equation}
Now let $\sum_{a \in A \setminus B} (\ef(\{a\} \cup B) - \ef(B)) = 0.$ This implies 
\begin{align*} 
    \ef(\{a\} \cup B) - \ef(B) &= 0 \quad \forall a \in A \setminus B \\
    \implies \max_{\theta_i \in F_{\theta_p}(B)} \xi_{pi} - \max_{\theta_i \in F_{\theta_p}(\{a\} \cup B)} \xi_{pi} &= 0 \quad \forall a \in A \setminus B \\
     \implies \max_{\theta_i \in F_{\theta_p}(B)} \xi_{pi} - \max_{\theta_i \in F_{\theta_p}(a) \cap F_{\theta_p}( B)} \xi_{pi} &= 0 \quad \forall a \in A \setminus B \\ \label{eq:lhs_bound}
   \implies  \max_{\theta_i \in F_{\theta_p}(B)}  \xi_{pi} = \max_{\theta_i \in F_{\theta_p}(a) \cap F_{\theta_p}( B)} \xi_{pi}   & \quad \quad \forall a \in A \setminus B.
\end{align*}
Let $\displaystyle\max_{ \theta_i \in F_{\theta_p}(B)} \xi_{pi} = \xi_{pq}$ for $\theta_q \in F_{\theta_p}(B)$. From Assumption 2, we have 
\begin{equation}
\displaystyle \max_{\theta_i \in F_{\theta_p}(a) \cap F_{\theta_p}( B)} \xi_{pi} = \xi_{pq} \quad \forall a \in A \setminus B,
\end{equation}
and it follows that $\theta_q \in F_{\theta_p}(a) \cap F_{\theta_p}( B), \forall a \in A \setminus B$. \newline

\noindent We have the following: $\displaystyle \ef(A \cup B) - \ef(B)= \max_{\theta_i \in F_{\theta_p}(B)} \xi_{pi} - \max_{\theta_i \in F_{\theta_p}(A \cup B)} \xi_{pi}$.\\

\noindent Using the fact that $A \cup B = (A \setminus B) \cup B$, we have $F_{\theta_p}(A \cup B) = F_{\theta_p}(A\setminus B) \cap F_{\theta_p}(B).$\\

\noindent Thus, we have
\begin{equation}
\begin{aligned}
    \label{eq:rhs_exp}
    &\displaystyle \ef(A \cup B) - \ef(B)=\\ & \max_{\theta_i \in F_{\theta_p}(B)} \xi_{pi} - \max_{\theta_i \in F_{\theta_p}(A\setminus B) \cap F_{\theta_p}(B)} \xi_{pi}.
\end{aligned}
\end{equation}
From the previous argument, we have $\theta_q \in F_{\theta_p}(a) \cap F_{\theta_p}( B), \forall a \in A \setminus B$ and $\displaystyle\max_{ \theta_i \in F_{\theta_p}(B)} \xi_{pi} = \xi_{pq}$ for $\theta_q \in F_{\theta_p}(B).$ Since $F_{\theta_p}(A \setminus B) = \displaystyle \bigcap_{a \in A \setminus B} F_{\theta_p}(a)$, we have $$F_{\theta_p}(A\setminus B) \cap F_{\theta_p}(B) = \left( \displaystyle \bigcap_{a \in A \setminus B} F_{\theta_p}(a) \right) \bigcap F_{\theta_p}(B).$$  This implies $\theta_q \in F_{\theta_p}(A\setminus B) \cap F_{\theta_p}( B).$ It directly follows that 
\begin{equation}
 \max_{\theta_i \in F_{\theta_p}(B)} \xi_{pi} = \max_{\theta_i \in F_{\theta_p}(A\setminus B) \cap F_{\theta_p}(B)} \xi_{pi} = \xi_{pq}.
\end{equation}
Thus, we have $\ef(A \cup B) - \ef(B) = 0$. This gives the trivial bound of $\gamma \le 1$ (since we define $0/0=1$ for characterizing $\gamma$ (see \cite{das2018approximate})). Therefore, in order to establish a non-trivial lower bound on $\gamma$, we consider $\sum_{a \in A \setminus B} (\ef(\{a\} \cup B) - \ef(B)) > 0.$\\

\noindent We now proceed to establish a non-trivial lower bound on $\sum_{a \in A \setminus B} (\ef(\{a\} \cup B) - \ef(B)).$
\begin{align}
\begin{split}
    &\sum_{a \in A \setminus B} (\ef(\{a\} \cup B) - \ef(B)) \\= &\sum_{a \in A \setminus B} \left(\max_{\theta_i \in F_{\theta_p}(B)} \xi_{pi} - \max_{\theta_i \in F_{\theta_p}(\{a\} \cup B)} \xi_{pi}\right)\\
     \label{eq:lb}
    & \geq \min_{\theta_i, \theta_j \in \Theta, \theta_i \ne \theta_j} |\xi_{pi} - \xi_{pj}|=\underline{\xi}_p.
    \end{split}
\end{align}

Next, we provide an upper bound on $\displaystyle \ef(A \cup B) - \ef(B).$ 
\begin{align}
\begin{split}
    \ef(A \cup B) - \ef(B) &=  \max_{\theta_i \in F_{\theta_p}(B)} \xi_{pi} - \max_{\theta_i \in F_{\theta_p}(A \cup B)} \xi_{pi} \\
    \label{eq:ub}
    & \leq \max_{\theta_i, \theta_j \in \Theta, \theta_i \ne \theta_j} |\xi_{pi} - \xi_{pj} |=\Bar{\xi}_p.
    \end{split}
\end{align}
Due to Assumption 2, we have $0< \underline{\xi}_p<1$ and $0< \Bar{\xi}_p < 1$.
Combining inequalities in \eqref{eq:lb} and \eqref{eq:ub}, we have that the function $f_{\theta_p}(\mathcal{I})$ is approximately submodular with a submodularity ratio $\gamma_p = \underline{\xi}_p/\Bar{\xi}_p$, for all $\theta_p \in \Theta$. 
\hfill   \end{proof}
We have the following result characterizing the set function ratios for the function $f_{\theta_p}(\cdot)$.
\begin{lemma}
\label{lma:ratios}
    Let $\check{\alpha}_p, \check{\nu}_p$ and $\kappa_p$ denote the inverse generalized curvature, superadditivity ratio and bipartite subadditivity ratio, respectively, for the function $f_{\theta_p}: 2^{\mathcal{D}} \to \mathbb{R}_{\ge 0}$. We have the following: \\
(i) $\displaystyle \check{\alpha}_p = 1-\frac{\min_{\theta_i, \theta_j \in \Theta, \theta_i \ne \theta_j} |\xi_{pi} - \xi_{pj}|}{\max_{\theta_i, \theta_j \in \Theta, \theta_i \ne \theta_j} |\xi_{pi} - \xi_{pj}|}, $\\
(ii) $\displaystyle \check{v}_p = \frac{1-\max_{\theta_j \in \Theta}\xi_{pj}}{|\mathcal{D}|(1-\min_{\theta_j \in \Theta}\xi_{pj})}$,\\
(iii) $\displaystyle \kappa_p = \frac{2 (1-\max_{\theta_j \in \Theta}\xi_{pj}) }{1-\min_{\theta_j \in \Theta}\xi_{pj}}$.
\end{lemma}
\begin{proof}
    We begin by proving part (i). Rearranging \eqref{eq:inverse_curvature}, we have
    \begin{equation}
    \begin{aligned}
    \label{eq:alpha}
    	 \check{\alpha} \ge 1-\frac{f(S) - f(S \setminus \{i\})}{f(S \cup V) - f((S\setminus \{i\}) \cup V)},\\ \quad \forall S, V \subseteq \Omega, i \in S \setminus V.
    \end{aligned}
    \end{equation}
    In order to find the smallest $\check{\alpha}$, we need to bound the ratio in \eqref{eq:alpha}. We will first begin by proving the following: 
    \begin{equation*}
        f(S \cup V) - f((S\setminus \{i\}) \cup V) = 0 \implies f(S) - f(S \setminus \{i\}) = 0.
    \end{equation*}
    From the definition of $f_{\theta_p}(\cdot)$, we have 
    \begin{align}
        \max_{\theta_j \in F_{\theta_p}((S\setminus \{i\}) \cup V)} \xi_{pj} &- \max_{\theta_j \in F_{\theta_p}(S \cup V)} \xi_{pj} = 0\\ 
        \label{eq:max_pen}
        \implies \max_{\theta_j \in F_{\theta_p}((S\setminus \{i\}) \cup V)} \xi_{pj} &= \max_{\theta_j \in F_{\theta_p}(S \cup V)} \xi_{pj}.
    \end{align}
     Let the maximum penalty in \eqref{eq:max_pen} correspond to some $\theta_q \in \Theta$. From Assumption 1, we have $\theta_q \in F_{\theta_p}((S \setminus \{i\}) \cup V)$ and $ \theta_q \in F_{\theta_p}(S \cup V)$. From \eqref{eq: obs_eq_intersect}, we have that the hypothesis $\theta_q$ corresponding to the maximum penalty satisfies $\theta_q \in F_{\theta_p}(S \setminus \{i\})$ and $\theta_q \in F_{\theta_p}(S)$. Since $f(S) - f(S \setminus \{i\}) = \max_{\theta_j \in F_{\theta_p}(S \setminus \{i\})} \xi_{pj} - \max_{\theta_j \in F_{\theta_p}(S)} \xi_{pj}$ and $i \in S \setminus V$, we have $f(S) - f(S \setminus \{i\}) = 0$. Therefore, we exclude the trivial scenario in which the ratio is not defined due to the form $0/0$. \\
     Additionally, we have from Lemma \ref{lma:approx_submod} that the function $f_{\theta_p}$ is weak-submodular. We know that a weak-submodular function with submodularity ratio $\gamma \in (0,1)$ satisfies $ f(A \cup \{i\}) - f(A) \ge \gamma \left( f(B \cup \{i\}) - f(B)\right)$ for all $A \subseteq B \subseteq \Omega$ and $i \in \Omega \setminus B$.
     That is, the marginal gain of adding an element to a smaller set is at least $\gamma$ times the marginal gain of adding the element to a larger set. By setting $A = S \setminus \{i\}$ and $B =  S \setminus \{i\} \cup V$, we have that the ratio in \eqref{eq:alpha} is bounded away from zero, and thus we have that $\check{\alpha} \ne 1$. We now proceed to characterize a non-trivial bound on the inverse curvature  $\check{\alpha}$. We lower bound the numerator as follows:
     \begin{align}
         f(S) - f(S\setminus \{i\}) &= \max_{\theta_j \in F_{\theta_p}(S \setminus \{i\})} \xi_{pj} - \max_{\theta_j \in F_{\theta_p}(S)} \xi_{pj}\\
         \label{eq:lbnum}
         & \ge \min_{\theta_i, \theta_j \in \Theta, \theta_i \ne \theta_j} |\xi_{pi} - \xi_{pj}|.
     \end{align}
     Next, we upper bound the denominator as follows:
     \begin{align}
          &f(S \cup V) - f(S\setminus \{i\} \cup V) = &\\
          &\max_{\theta_j \in F_{\theta_p}(S \setminus \{i\} \cup V)} \xi_{pj} - \max_{\theta_j \in F_{\theta_p}(S \cup V)} \xi_{pj}\\ 
          \label{eq:ubden}
          & \le \max_{\theta_i, \theta_j \in \Theta, \theta_i \ne \theta_j} |\xi_{pi} - \xi_{pj}|.
     \end{align}
   Combining \eqref{eq:lbnum} and \eqref{eq:ubden}, we obtain the result of Part (i). We note the following: separately bounding the numerator and denominator of a ratio is valid but can lead to loose bounds, as it neglects shared dependencies. However, one can derive instance-dependent bounds, which are often tighter, by incorporating these relationships for the given instance of the problem. Here, we present generalized bounds that hold for any instance of the R-MPIS problem.
   
         Now consider the superadditivity ratio defined in \eqref{eq:superadditivity}. In order to find the largest $\check{\nu}$, we have the following: 
         \begin{align*}
             \min (f(S)) = \min_{S \subseteq \mathcal{D}} (1- \max_{\theta_j \in F_{\theta_p}(S)}\xi_{pj}) \ge  1- \max_{\theta_j \in \Theta} \xi_{pj}\\
             \max (\sum_{i \in S} f(\{i\})) =  \max_{S \subseteq \mathcal{D}} (\sum_{i \in S} (1- \max_{\theta_j \in F_{\theta_p}(i)} \xi_{pj}) \\ \quad \le |S| (1- \min_{\theta_j \in \Theta} \xi_{pj}).
         \end{align*}
         Since we have $|S| \le |\mathcal{D}|$, setting $|S| = |\mathcal{D}|$ establishes the bound in Part (ii). 
         
         Finally, we consider the Bipartate subadditivity ratio defined in \eqref{eq:2_separate_sa} and establish a bound for $\kappa$ as follows.
         \begin{align}
             f(A) + f(B) &= 1 - \max_{\theta_j \in F_{\theta_p}(A)} \xi_{pj} + 1 - \max_{\theta_j \in F_{\theta_p}(B)} \xi_{pj}\\
             \label{eq:ubkappa}
            \implies f(A) + f(B) & \ge 2(1- \max_{\theta_j \in \Theta} \xi_{pj})\\
            f(S) & = 1 - \max_{\theta_j \in F_{\theta_p}(S)} \xi_{pj}\\
            \label{eq:lbkappa}
           \implies f(S) & \le 1 - \min_{\theta_j \in \Theta} \xi_{pj}. 
         \end{align}
        Combining \eqref{eq:ubkappa} and \eqref{eq:lbkappa}, we obtain the result in Part (iii).
\hfill   \end{proof}
\begin{remark}
\label{remark:kappa}
    The ratios $\check{\alpha}_p, \check{\nu}_p$ and $\kappa_p$ must be bounded in $(0,1)$ to obtain non-trivial approximation guarantees. We have $\check{\alpha}_p$ and $\check{\nu}_p$ in $(0,1)$, however, for $\kappa_p$ to be in $(0,1)$, the misclassification penalties must satisfy: $1-\min_{\theta_j \in \Theta} \xi_{pj} \ge 2( 1-\max_{\theta_j \in \Theta} \xi_{pj})$. We note that this is somewhat a restrictive assumption and that the guarantees presented will only hold for R-MPIS problem instances in which the misclassification penalties satisfy this condition.
\end{remark}

Let $\check{\alpha}, \displaystyle\check{\nu}, \kappa$ and $\gamma$ denote the inverse generalized curvature, superadditivity ratio, bipartite subadditivity ratio, and submodularity ratio for the objective function $\Lambda(\cdot)$ of the R-MPIS problem, respectively. We have the following result, which characterizes the curvature properties of $\Lambda(\cdot)$.

\begin{proposition}
\label{lma:overall_ratios}
The objective function of the R-MPIS problem $\Lambda(\mathcal{I}) = \sum_{\theta_p \in \Theta} f_{\theta_p}(\mathcal{I})$ has the following properties:\\
    (i) $\displaystyle\check{\alpha} = \max_{\theta_p \in \Theta} \check{\alpha}_p$;\\
    (ii) $\displaystyle\check{\nu} = \min_{\theta_p \in \Theta} \check{\nu}_p$;\\
    (iii) $\displaystyle\kappa = \min_{\theta_p \in \Theta} \kappa_p$;\\
    (iv) $\gamma = \displaystyle \min_{\theta_p \in \Theta} \gamma_p$.
\end{proposition}
\begin{proof}
   Let \( \mathcal{\mathcal{F}} = f_1 + f_2 + \dots + f_m \) be the sum of \( m \) monotone non-decreasing set functions \( f_1, f_2, \dots, f_m \). We analyze the behavior of the following properties for \( \mathcal{F} \): the inverse generalized curvature, the superadditivity ratio, the bipartite subadditivity ratio, and the submodularity ratio. As a shorthand, we denote $f(\{a\} \mid A) = f(A \cup \{a\}) - f(A)$.
   
 The inverse generalized curvature \( \tilde{\alpha}_j \) of a function \( f_j \) is defined as the smallest scalar \( \tilde{\alpha}_j \in [0, 1] \) such that:
\begin{equation}
\frac{f_j(\{i\} \mid S \setminus \{i\})}{f_j(\{i\} \mid S \setminus \{i\} \cup V)} \geq 1 - \tilde{\alpha}_j, \quad \forall S, V \subseteq \Omega, i \in S \setminus V.
\end{equation}
By definition of \( \mathcal{F} \), the marginal contributions satisfy:
\begin{align}
\mathcal{F}(\{i\} \mid S \setminus \{i\}) &= \sum_{j=1}^m f_j(\{i\} \mid S \setminus \{i\}), \\
\mathcal{F}(\{i\} \mid S \setminus \{i\} \cup V) &= \sum_{j=1}^m f_j(\{i\} \mid S \setminus \{i\} \cup V).
\end{align}
Using the individual inverse generalized curvatures \( \tilde{\alpha}_j \), we have:
\begin{equation}
f_j(\{i\} \mid S \setminus \{i\}) \geq (1 - \tilde{\alpha}_j) f_j(\{i\} \mid S \setminus \{i\} \cup V).
\end{equation}
Summing over all \( j \), we get:
\begin{equation}
\sum_{j=1}^m f_j(\{i\} \mid S \setminus \{i\}) \geq \sum_{j=1}^m (1 - \tilde{\alpha}_j) f_j(\{i\} \mid S \setminus \{i\} \cup V).
\end{equation}
Factoring out the largest curvature ratio \( \tilde{\alpha}_{\max} = \max_{j} \tilde{\alpha}_j \), we obtain:
\begin{equation}
\frac{\mathcal{F}(\{i\} \mid S \setminus \{i\})}{\mathcal{F}(\{i\} \mid S \setminus \{i\} \cup V)} \geq 1 - \tilde{\alpha}_{\max}.
\end{equation}
Thus, \( \tilde{\alpha}_\mathcal{F} = \max_{j} \tilde{\alpha}_j \).

The superadditivity ratio \( \tilde{\nu}_j \) of \( f_j \) is the largest scalar \( \tilde{\nu}_j \in [0, 1] \) such that:
\begin{equation}
\frac{f_j(S)}{\sum_{i \in S} f_j(\{i\})} \geq \tilde{\nu}_j, \quad \forall S \subseteq \Omega.
\end{equation}
By definition of \( \mathcal{F} \), we know:
\begin{align}
\mathcal{F}(S) &= \sum_{j=1}^m f_j(S), \\
\sum_{i \in S} \mathcal{F}(\{i\}) &= \sum_{j=1}^m \sum_{i \in S} f_j(\{i\}).
\end{align}
Using the individual superadditivity ratios \( \tilde{\nu}_j \), we have:
\begin{equation}
f_j(S) \geq \tilde{\nu}_j \sum_{i \in S} f_j(\{i\}).
\end{equation}
Summing over all \( j \), we get:
\begin{equation}
\sum_{j=1}^m f_j(S) \geq \sum_{j=1}^m \tilde{\nu}_j \sum_{i \in S} f_j(\{i\}).
\end{equation}
The minimum \( \tilde{\nu}_j \) dominates, giving:
\begin{equation}
\frac{\mathcal{F}(S)}{\sum_{i \in S} \mathcal{F}(\{i\})} \geq \min_{j} \tilde{\nu}_j.
\end{equation}
Thus, \( \tilde{\nu}_\mathcal{F} = \min_{j} \tilde{\nu}_j \).

The bipartite subadditivity ratio \( \kappa_j \) of \( f_j \) is the largest scalar \( \kappa_j \in [0, 1] \) such that:
\begin{equation}
\frac{f_j(A) + f_j(B)}{f_j(S)} \geq \kappa_j, \quad \forall A, B \subseteq V, A \cup B = S, A \cap B = \emptyset.
\end{equation}
By definition of \( \mathcal{F} \), we know:
\begin{align}
\mathcal{F}(A) &= \sum_{j=1}^m f_j(A), \hspace{2pt} \mathcal{F}(B) = \sum_{j=1}^m f_j(B), \hspace{2pt} \mathcal{F}(S) = \sum_{j=1}^m f_j(S),
\end{align}
where \( S = A \cup B \) and \( A \cap B = \emptyset \). Using the bipartite subadditivity ratios \( \kappa_j \), we have:
\begin{equation}
f_j(A) + f_j(B) \geq \kappa_j f_j(S).
\end{equation}
Summing over all \( j \), we get:
\begin{equation}
\sum_{j=1}^m \left( f_j(A) + f_j(B) \right) \geq \sum_{j=1}^m \kappa_j f_j(S).
\end{equation}
Factoring out the smallest \( \kappa_j \), we obtain:
\begin{equation}
\frac{\mathcal{F}(A) + \mathcal{F}(B)}{\mathcal{F}(S)} \geq \min_{j} \kappa_j.
\end{equation}
Thus, \( \kappa_\mathcal{F} = \min_{j} \kappa_j \).
Now let $\gamma_j$ be the submodularity ratio of $f_j$. Based on similar arguments made for the superadditivty ratio, we have 
\begin{equation} \displaystyle \gamma_\mathcal{F} = \min_{j} \gamma_j.
\end{equation}
We let $f_j = f_{\theta_p}(\cdot)$, $\check{\alpha}_j = \check{\alpha}_p$, $\check{\nu}_j = \check{\nu}_p$, $\kappa_j = \kappa_p$ and $\gamma_j = \gamma_p$ and $\mathcal{F} = \Lambda(\cdot) $. Therefore, we  have     (i) $\displaystyle\check{\alpha} = \max_{\theta_p \in \Theta} \check{\alpha}_p$;
    (ii) $\displaystyle\check{\nu} = \min_{\theta_p \in \Theta} \check{\nu}_p$;
    (iii) $\displaystyle\kappa = \min_{\theta_p \in \Theta} \kappa_p$; and (iv) $\displaystyle\gamma = \min_{\theta_p \in \Theta} \gamma_p$.
\hfill   \end{proof}
\subsection{Robust Greedy Algorithm}
 We adapt the oblivious greedy algorithm presented in \cite{bogunovic2018robust} for the R-MPIS problem. The algorithm requires a non-negative monotone set function $f:2^{\Omega} \to \mathbb{R}_{+}$, and the ground set of items $\Omega$. Here the objective function is $f(\cdot) = \Lambda(\cdot)$ and the ground set is the set of available information sources $\Omega = \mathcal{D}$. We recall that $K$ and $A$ are the selection and attack budgets, respectively. The algorithm constructs two sets $S_0$ and $S_1$. The first set $S_0$ is constructed via oblivious selection, i.e.~$\lceil \beta A \rceil$ items with the individually highest objective values are selected. Here, $\beta \in \mathbb{R}_+$ is a hyperparameter that together with $A$, determines the size of $S_0$ ($|S_0| = \lceil \beta A \rceil \leq K$). The choice of parameter $\beta$ depends on the specific instance of the R-MPIS problem, and one must have $\beta \in (1,K/A)$ in order to obtain non-trivial performance bounds. We provide more information about this parameter later in this section. The second set $S_1$, of size $K - |S_0|$, is obtained by running the standard greedy algorithm (e.g., Algorithm 1 in \cite{bhargav2023complexity}), which we denote as \textsc{Greedy}, on the remaining items $\mathcal{D} \setminus S_0$. The \textsc{Greedy} sub-routine takes the ground set $\mathcal{D}\setminus S_0$ and returns a set $S_1$ of size at most $K - |S_0|$, with the largest utility.  Finally, the algorithm outputs the set $\mathcal{I} = S_0 \cup S_1$ of size $K$ that is robust against the worst-case removal of $A$ elements.

\begin{algorithm}[htpb!]
    \caption{Robust Greedy Algorithm for R-MPIS \label{algorithm:alg}}
    \begin{algorithmic}[1]
      \Require   $\mathcal{D}$, $K$, $A$, $\beta \in \mathbb{R}_+ \text{ s.t. } \lceil \beta A \rceil \leq K$
       \Ensure Set $\mathcal{I} \subseteq \mathcal{D}$ such that $|\mathcal{I}| \leq K$
            \State $S_0, S_1 \gets \emptyset$
            \For {$ i \gets 0 \textbf{ to } 	\lceil \beta A \rceil$}
                \State  $ v \gets \argmax_{v \in \mathcal{D} \setminus S_0} \Lambda(\lbrace v \rbrace)$
                \State $S_0 \gets S_0 \cup \lbrace v \rbrace$
            \EndFor
        \State $S_1 \gets \textsc{Greedy} (K - |S_0|,\ (\mathcal{D} \setminus S_0))$
    \State {$\mathcal{I}\gets S_0 \cup S_1$}\\
    \Return $\mathcal{I}$ 
    \end{algorithmic}
\end{algorithm}

Based on the results presented in Theorem 1 in \cite{bogunovic2018robust}, we have the following theoretical guarantees for Algorithm \ref{algorithm:alg} applied to the R-MPIS problem.

\begin{theorem}\label{thm:main_thm}
 For an instance of the R-MPIS problem, let $\mathcal{I}^* $and $\mathcal{I}'^{*}$ denote the optimal (robust) information set and attack set respectively. Let $\mathcal{I}$ and $\mathcal{I}'$ denote the information set selected by Algorithm \ref{algorithm:alg} and the corresponding attack set, respectively.  Algorithm \ref{algorithm:alg} with the parameter $\beta $ such that $\lceil \beta A \rceil \leq K$ and $\beta > 1$, returns a set of information sources $\mathcal{I} \subseteq \mathcal{D}$ of size $K$ such that
\[	
	\Lambda(\mathcal{I}\setminus \mathcal{I'}) \ge \frac{\kappa P \left(1-e^{-\gamma\frac{1 -   \beta c }{1- c}}\right)}{1 + P\left(1-e^{-\gamma\frac{1 -   \beta c }{1-c}}\right)} \Lambda(\mathcal{I}^* \setminus \mathcal{I}'^{*}),
\]
where $\displaystyle P = \frac{(\beta-1)\check{\nu}(1-\check{\alpha})}{1+(\beta-1)\check{\nu}(1-\check{\alpha})}$,  $\gamma \in (0,1)$ is the submodularity ratio, $\kappa \in (0,1)$ is the bipartite subadditivity ratio,  $\check{\alpha} \in (0,1)$ is the inverse curvature, $\check{\nu} \in (0,1)$ is the superadditivity ratio of the objective function $\Lambda(\cdot)$ and $c \in (0,1)$ is a constant such that $A = \lceil cK \rceil$.
\end{theorem}

We now discuss the effect of the problem parameters on different curvature ratios and the parameter $\beta$, and consequently on the performance guarantees provided by Theorem \ref{thm:main_thm} for the R-MPIS problem.
\begin{itemize}
    \item Effect of $\kappa$: From Lemma \ref{lma:ratios}, we see that if the maximum penalty of misclassification is higher for some hypothesis (i.e., $\xi_{ij} \to 1$ for some $\theta_i, \theta_j \in \Theta$), then $\kappa \to 0$ and the performance bound degrades. Additionally, only instances of R-MPIS problem that satisfy $\kappa <1$ admit non-trivial guarantees (see Remark \ref{remark:kappa}).
    \item Effect of $\check{\alpha}, \check{\nu}, \text{and } \gamma$: From Lemma \ref{lma:approx_submod} and Lemma \ref{lma:ratios}, we see that if the misclassification penalties are close to each other, then $\gamma \to 0$ and $\check{\alpha} \to 1$. Also, if the number of information sources to select from increases (i.e, if $\mathcal{D}$ is large), then $\check{\nu}$ decreases. In either of the above scenarios, $P \to 0$ and the performance bounds become weaker.
 \item Effect of $\beta$: If $\beta$ takes the maximum value of $\lfloor K/A \rfloor$, then $\beta c \to 1$. Furthermore, if the fraction of information sources that can be attacked approaches $1$, i.e., if $A$ is closer to $K$, then $\beta \to 1$ and $P \to 0$. In either of the above scenarios, the performance bound weakens. In practice, one can determine the appropriate value of $\beta$ that maximizes the bound and provides a better performance through empirical studies.
\end{itemize}

In many practical scenarios, the submodularity ratio $\gamma$ of the maximum penalty metric can be arbitrarily small (or zero) when the misclassification penalties for two hypotheses are very close to each other (or equal). It is also easy to verify that the submodularity ratio $\gamma$ decreases as the number of hypotheses increase. Furthermore, the ratios $\kappa$ and $P$ can be close to zero in many scenarios where the misclassification penalties are close to each other. As a result, performance guarantees may deteriorate significantly and become arbitrarily small. In such scenarios, one can turn to an alternate metric, which can provide non-trivial performance guarantees for greedy approximation algorithms. To this end, we present an alternate (surrogate) metric and a robust greedy algorithm with improved performance guarantees in Section \ref{sec:alt_form}.
\newpage
\section{Alternate Metric for Information Selection}
\label{sec:alt_form}
 In this section, we present an alternate (surrogate) metric to characterize the quality of an information set, based on the total penalty of misclassification. 
 
 The total penalty of misclassification is defined as follows: 
\begin{equation}
    \label{eq:tot_pen}
    \rho_{\theta_p}(\mathcal{I}) = \sum_{\theta_i \in F_{\theta_p}(\mathcal{I})} \xi_{pi}.
\end{equation}
Intuitively, in order to minimize the total penalty ($\rho_{\theta_p}(\mathcal{I})$), one has to select a subset $\mathcal{I} \subseteq \mathcal{D}$ that ensures that the number of hypotheses which are observationally equivalent to the true hypothesis $\theta_p$, i.e., $|F_{\theta_p}(\mathcal{I})|$, is small and/or the hypotheses that are observationally equivalent to the true hypothesis have lower misclassification penalties. Effectively, this may result in a lower penalty associated with misclassifying the true hypothesis. We now formalize the Modified Robust Minimum Penalty Information Set Selection (M-RMPIS) Problem based on this surrogate metric as follows. 

\begin{Problem}[M-RMPIS]
\label{prob:mmpis}
    Consider a set of possible states of the world $\Theta=\left\{\theta_1, \ldots, \theta_m\right\}$; a set $\mathcal{D}$ of information sources; a row-stochastic penalty matrix $ \Xi = [\xi_{ij}] \in \mathbb{R}^{m \times m}$ ; a selection budget $K \in \mathbb{R}_{\geq0}$ and an attack budget $A \in \mathbb{R}_{\ge 0}$. The M-RMPIS Problem is to find a set of selected information sources $\mathcal{I} \subseteq \mathcal{D}$ that solves
\begin{equation}
\label{eq:mmpis}
\begin{aligned}
 \min _{\mathcal{I} \subseteq \mathcal{D}}  \max _{\mathcal{I} \subseteq \mathcal{I}'} & \hspace{5pt} \displaystyle \sum_{\theta_p \in \Theta} \rho_{\theta_p}(\mathcal{I} \setminus \mathcal{I}') ;
  \hspace{2pt} \text { s.t. } |\mathcal{I}|\leq K; |\mathcal{I}'|\leq A.
\end{aligned}
\end{equation}
\end{Problem}
Consider the following optimization problem: 
\begin{equation}
\label{eq:rmpis_max}
\begin{aligned}
 \max _{\mathcal{I} \subseteq \mathcal{D}} \min _{\mathcal{I} \subseteq \mathcal{I}'} & \hspace{5pt} \displaystyle \sum_{\theta_p \in \Theta} \left(1- \rho_{\theta_p}(\mathcal{I} \setminus \mathcal{I}') \right);
  \hspace{2pt}\\ \textit { s.t. } & |\mathcal{I}|  \leq K, |\mathcal{I}'|\leq A.
\end{aligned}
\end{equation}
It is easy to verify that the problem defined in (\ref{eq:rmpis_max}) is equivalent to the problem defined in (\ref{eq:mmpis}), i.e., the information set $\mathcal{I} \subseteq \mathcal{D}$ that optimizes the problem in Equation (\ref{eq:rmpis_max}) is also the optimal solution to the Problem \ref{prob:mmpis}. We denote $g_{\theta_p}= 1 -\rho_{\theta_p}(\mathcal{I})$ and the objective function by $\Gamma(\mathcal{I} \setminus \mathcal{I}') = \sum_{\theta_p \in \Theta}g_{\theta_p}(\mathcal{I} \setminus \mathcal{I}')$.

\subsection{Robust Submodular Information Selection}
In this section, we establish the submodularity and curvature properties of the objective function for the M-RMPIS problem, which we will use for efficient approximation. Specifically, we adapt the greedy algorithm presented in \cite{tzoumas2018resilient} for robust maximization of submodular functions (see Remark \ref{remark:robust_submod}). We characterize its performance guarantees for the robust information set selection problem under the surrogate metric. 

\begin{algorithm}[htpb!]
    \caption{Robust Greedy Algorithm for M-RMPIS \label{algorithm:alg2}}
    \begin{algorithmic}[1]
      \Require  $\mathcal{D}$, $K$, $A$
       \Ensure Set $\mathcal{I} \subseteq \mathcal{D}$ such that $|\mathcal{I}| \leq K$
            \State $\mathcal{A}_1 \leftarrow \emptyset ; \quad \mathcal{R}_1 \leftarrow \emptyset ; \quad \mathcal{A}_2 \leftarrow \emptyset ; \quad \mathcal{R}_2 \leftarrow \emptyset ;$
            \While {$\mathcal{R}_1 \neq \mathcal{D}$}
            \State $x \in \arg \max _{y \in \mathcal{D} \backslash \mathcal{R}_1} \Gamma(y)$;
             \If {$|\mathcal{A}_1 \cup\{x\}| \le A$}
            \State $\mathcal{A}_1 \leftarrow \mathcal{A}_1 \cup\{x\} ;$
            \EndIf
            \State $\mathcal{R}_1 \leftarrow \mathcal{R}_1 \cup\{x\} ;$
            \EndWhile 
            \While {$\mathcal{R}_2 \neq \mathcal{D}\setminus \mathcal{A}_1$}
            \State $x \in \arg \max _{y \in \mathcal{D} \backslash (\mathcal{A}_1 \cup \mathcal{R}_2)} \Gamma(\mathcal{A}_2 \cup \{y\})$;
             \If {$|\mathcal{A}_1 \cup \mathcal{A}_2 \cup \{x\}| \le K$}
            \State $\mathcal{A}_2 \leftarrow \mathcal{A}_2 \cup\{x\} ;$
            \EndIf
            \State $\mathcal{R}_2 \leftarrow \mathcal{R}_2 \cup\{x\} ;$
            \EndWhile \\
    \Return $\mathcal{I} = \mathcal{A}_1 \cup \mathcal{A}_2$ 
    \end{algorithmic}
\end{algorithm}

We present the greedy approximation algorithm in Algorithm \ref{algorithm:alg2}. In Algorithm \ref{algorithm:alg2}, the set $\mathcal{A}_1$ approximates worst-case set removal from $\mathcal{I}$. The set $\mathcal{A}_2$ is such that the set $\mathcal{A}_1 \cup \mathcal{A}_2$ approximates the optimal solution to the M-RMPIS problem. Assuming that $\mathcal{A}_1$ is the set that will be removed from $\mathcal{I}$, Algorithm \ref{algorithm:alg2} needs to select a set of elements $\mathcal{A}_2$ to complete the construction of $\mathcal{I}$ such that $|\mathcal{I}| \le K$. The sets $\mathcal{R}_1$ and $\mathcal{R}_2$ aid in the construction of $\mathcal{A}_1$ and $\mathcal{A}_2$, which keep track of elements that are remaining to be checked and whether they could be added or not. We refer interested readers to Section 3 of \cite{tzoumas2018resilient} for more detailed discussions.

We will now proceed to establish the performance guarantees for Algorithm \ref{algorithm:alg2}. First, we have the following result, establishing the submodularity property of $g_{\theta_p}(\cdot)$. 

\begin{lemma}
    \label{lma:surr}
    The function $g_{\theta_p}(\mathcal{I}):2^{\mathcal{D}}\to \mathbb{R}$ is submodular for all $\theta_p \in \Theta$.
\end{lemma}
\begin{proof}
    Recall that $g_{\theta_p}(\mathcal{I}) = 1 - \sum_{\theta_i \in F_{\theta_p}(\mathcal{I})} \xi_{pi}.$ Consider any $\mathcal{I}_1 \subseteq \mathcal{I}_2 \subseteq\mathcal{D}$ and any $j \in \mathcal{D} \backslash \mathcal{I}_2$. We have the following:
{\allowdisplaybreaks
\begin{align*}
 & g_{\theta_p}\left(\mathcal{I}_1 \cup\{j\}\right) - g_{\theta_p}\left(\mathcal{I}_1\right)\\ 
= & \sum_{\theta_i \in F_{\theta_p}\left(\mathcal{I}_1\right)} \xi_{pi} - \sum_{\theta_i \in F_{\theta_p}\left(\mathcal{I}_1 \cup\{j\}\right)}  \xi_{pi}  \\
= & \sum_{\theta_i \in F_{\theta_p}\left(\mathcal{I}_1\right)} \xi_{pi} -\sum_{\theta_i \in F_{\theta_p}\left(\mathcal{I}_1\right) \cap F_{\theta_p}(j)} \xi_{pi} \\
= & \sum_{\theta_i \in  F_{\theta_p}\left(\mathcal{I}_1\right) \backslash \left(F_{\theta_p}\left(\mathcal{I}_1\right) \cap F_{\theta_p}(j)\right)} \xi_{pi}\\
=&\sum_{\theta_i \in F_{\theta_p}(\mathcal{I}_1) \backslash F_{\theta_p}\left(j\right)} \xi_{pi} .
\end{align*}
}
Note that the above arguments follow from  De Morgan's laws. Similarly, we  have
\begin{equation*}
g_{\theta_p}\left(\mathcal{I}_2 \cup\{j\}\right)- g_{\theta_p}\left(\mathcal{I}_2\right)=\sum_{\theta_i \in F_{\theta_p}(\mathcal{I}_2) \backslash F_{\theta_p}\left(j\right)} \xi_{pi} .
\end{equation*}
Since $\mathcal{I}_1 \subseteq \mathcal{I}_2$, we have $F_{\theta_p}(\mathcal{I}_2) \backslash F_{\theta_p}\left(j\right) \subseteq F_{\theta_p}(\mathcal{I}_1) \backslash F_{\theta_p}\left(j\right)$. Thus,
\begin{equation*}
g_{\theta_p}\left(\mathcal{I}_1 \cup\{j\}\right) - g_{\theta_p}\left(\mathcal{I}_1\right) \geq g_{\theta_p}\left(\mathcal{I}_2 \cup\{j\}\right)-g_{\theta_p}\left(\mathcal{I}_2\right) .
\end{equation*}
Since the above arguments hold for all $\theta_p \in \Theta$, the function $g_{\theta_p}(\cdot)$ is submodular for all $\theta_p \in \Theta$.
\hfill   \end{proof}
\begin{definition}[Curvature of Submodular Functions]
\label{def:curv}
    Consider a finite set $\Omega$, and a non-decreasing submodular set function $g:2^\Omega \mapsto\mathbb{R}$ such that (without loss of generality) for any element $e \in \Omega$, it is  $g(e)\neq 0$.  Then, the curvature of $g$ is defined as follows: \begin{equation}\label{eq:curvature}
c_g =  1-\min_{e\in\Omega}\frac{g(\Omega)-g(\Omega\setminus\{e\})}{g(e)}.
\end{equation}
\end{definition}
Definition \ref{def:curv} implies that for any non-decreasing submodular set function $g:2^{\Omega} \to \mathbb{R}$, it holds that $0 \le c_g \le 1$. The value $c_g$ measures how far $g$ is from modularity. If $c_g = 0$, then the function is modular. If $c_g=1$, then there exists an element $e \in \Omega$ such that $g(\Omega) = g(\Omega \setminus \{e\})$, i.e., in the presence of $\Omega \setminus \{e\}$, $e$ loses all its contribution. Furthermore, the generalized curvature $\alpha$ (defined in \cite{vondrak1978submodularity}) reduces to $c_g$ for submodular functions. 

We now have the following result, which characterizes the curvature of $g_{\theta_p}(\cdot)$ and the objective function $\Gamma(\cdot)$ of the M-RMPIS problem.

\begin{lemma}
\label{lma:}
Let $c_{gp}$ denote the curvature of $g_{\theta_p}:2^{\mathcal{D}} \to \mathbb{R}$ and let $c_{\Gamma}$ denote the curvature of the function $\Gamma(\mathcal{I}) = \sum_{\theta_p \in \Theta} g_{\theta_p}(\mathcal{I})$. We have
  \begin{equation}
        c_{\Gamma} = \max_{\theta_p \in \Theta} c_{gp},
        \end{equation}
        where 
    \begin{equation}
    \label{eq:cgp}
     c_{gp} = 
    \begin{cases}
      \quad  \quad \quad \quad 1  & \text{if for some  $ i \in \mathcal{D}$} \\ & g_{\theta_p}(\mathcal{D}) = g_{\theta_p}(\mathcal{D}\setminus i), \\
      \displaystyle 1- \frac{\min_{\theta_i, \theta_j \in \Theta}|\xi_{pi} - \xi_{pj}|}{\max_{\theta_j \in \Theta} \xi_{pj}}& \text{otherwise}. 
    \end{cases}
    \end{equation}

\end{lemma}

\begin{proof}
    We first begin by characterizing $c_{gp}$. From \eqref{eq:curvature}, we have 
    \begin{equation}
        c_{gp} = 1 - \min_{e \in \mathcal{D}}\frac{\max_{\theta_j \in F_{\theta_p}(\mathcal{D} \setminus \{e\})} \xi_{pj} - \max_{\theta_j \in F_{\theta_p}(\mathcal{D})} \xi_{pj} }{\max_{\theta_j \in F_{\theta_p}(e)} \xi_{pj}}
    \end{equation} 
    Suppose that there exists an information source $e \in \mathcal{D}$ such that $F_{\theta_p}(\mathcal{D} \setminus \{e\}) = F_{\theta_p}(\mathcal{D})$ for some $\theta_p \in \Theta$, we have $\max_{\theta_j \in F_{\theta_p}(\mathcal{D} \setminus \{e\})} \xi_{pj} = \max_{\theta_j \in F_{\theta_p}(\mathcal{D})} \xi_{pj}$. It follows that $g_{\theta_p}(\mathcal{D}) = g_{\theta_p}(\mathcal{D}\setminus \{e\})$, and $c_{gp}=1$. 
    
    Now, if the above condition does not hold, we proceed to establish a bound on the minimum value of the fraction over all $e \in \mathcal{D}$. 
    \begin{equation}
    \label{eq:curv_lb}
        \max_{\theta_j \in F_{\theta_p}(\mathcal{D} \setminus \{e\})} \xi_{pj} - \max_{\theta_j \in F_{\theta_p}(\mathcal{D})} \xi_{pj} \ge \min_{\theta_i, \theta_j \in \Theta}|\xi_{pi} - \xi_{pj}|.
    \end{equation}

    \begin{equation}
     \label{eq:curv_ub}
        \max_{\theta_j \in F_{\theta_p}(e)} \xi_{pj} \le \max_{\theta_j \in \Theta} \xi_{pj}.
    \end{equation}
    Combining \eqref{eq:curv_lb} and \eqref{eq:curv_ub}, we obtain the expression for $c_{gp}$. 

    Based on similar arguments as in Proposition \ref{lma:overall_ratios} for the inverse generalized curvature $\check{\alpha}$, we have $c_{\Gamma} = \displaystyle \max_{\theta_p \in \Theta} c_{gp}$.
\hfill   \end{proof}

Define $h(K,A) = \max [1/(1+A), 1/(K-A)]$. Based on Theorem 1 in \cite{tzoumas2018resilient}, we have the following performance guarantees for Algorithm \ref{algorithm:alg2}. 

\begin{theorem}
\label{thm:surr_opt}
    Consider an instance of the M-RMPIS problem. Let $\Gamma^*$ denote the optimal objective function value. Let $\mathcal{I}$ be the information set returned by Algorithm \ref{algorithm:alg2} and let $\mathcal{I}' \subseteq \mathcal{I}$ be the set of information sources removed by the attacker (i.e., optimal attack set) from $\mathcal{I}$. Algorithm \ref{algorithm:alg2} has the following guarantees:
    \begin{equation}
        \Gamma(\mathcal{I} \setminus \mathcal{I}') \ge \frac{\max [1-c_{\Gamma}, h(K,A)]}{c_{\Gamma}} (1-e^{-c_{\Gamma}}) \hspace{2pt} \Gamma^*.
    \end{equation}
\end{theorem}

\textbf{Note:} For a finite selection budget $K$, the value of $h(K,A)$ is always non-zero. The minimum value of $h(K,A)$ is $2/(K+2)$ (which occurs at $A = K/2$) and the maximum value is $1$. For a fixed $K$, $h(K,A)$ is decreasing with respect to $A$ in the interval $0\le  A \le K/2$ and increasing with respect to $A$ in the interval $K/2 +1 \le A < K$.  

\begin{remark}
\label{remark:robust_submod}
    In general, for many instances of the M-RMPIS problem, the curvature is $c_{\Gamma} = 1$. Then the near-optimality guarantees provided by Theorem \ref{thm:surr_opt} only depend on $h(K,A)$ and is given by $\Gamma(\mathcal{I} \setminus \mathcal{I}') \ge h(K,A) (1-e^{-1}) \hspace{2pt} \Gamma^* $. Furthermore, one can employ Algorithm \ref{algorithm:alg} to obtain an approximate solution to the M-RMPIS problem, where the theoretical guarantees presented in Theorem \ref{thm:main_thm} hold with $\gamma = 1$ (since for submodular functions, the submodularity ratio $\gamma$ satisfies $\gamma = 1$). However, Algorithm \ref{algorithm:alg2} provides stronger guarantees (as in Theorem \ref{thm:surr_opt}).
\end{remark}

\section{Empirical Evaluation}
In this section, we empirically demonstrate the performance of the proposed algorithms on several instances of the information selection problem.
\vspace{-0.3cm}
\subsection{Case Study: Robust Threat Detection in Naval Fleet Surveillance}
We consider a naval fleet surveillance system tasked with detecting and classifying potential maritime threats in a naval operation zone. The task is to classify observed objects into one of the following 10 threat categories: $\Theta=$ \{\textit{Warship, Submarine, Fishing Vessel, Cargo Ship, Tanker, Speedboat, Drone Swarm, Pirate Skiff, Floating Debris, Background Noise}\}. We refer to this as the Maritime Threat Detection (MTD) Task. Misclassifying hostile entities, such as a warship or a submarine, as benign objects (e.g., cargo ship or background noise) can result in significant security breaches, while unnecessary escalation due to false positives (e.g., classifying fishing vessels as pirate skiffs) can waste critical resources. 

\textbf{Misclassification Penalty Matrix}: The \textit{penalty matrix} \( \Xi \) quantifies the cost to pay (penalty) for misclassifying different entities. The penalty values are defined by the system designer and are shaped by factors such as resource costs, availability, and the system's capacity to tolerate misclassification events. The penalty matrix is illustrated in Figure \ref{fig1}.  Each row of the matrix is normalized to ensure consistency in scaling. 
\begin{itemize}
    \item \textit{High penalties} are assigned for misclassifying critical threats as benign, such as detecting a \textit{submarine} or \textit{drone swarm} as \textit{background noise} (e.g., \( \Xi(\textit{Warship}, \textit{Background Noise}) = 22 \)).  
    \item \textit{Moderate penalties} are assigned when benign objects such as \textit{ fishing vessels} or \textit{floating debris} are classified as threats (e.g., \( \Xi(\textit{Fishing Vessel}, \textit{Pirate Skiff}) = 15 \)).  
    \item \textit{Lower penalties} are assigned for misclassification events that occur between benign objects or between objects with similar threats, as they require similar resources to secure the facility 
    (e.g., \( \Xi(\textit{Speed Boat}, \textit{Background Noise}) = 5 \)).
\end{itemize}

\textbf{Information Sources:}  
    The system consists of \( |\mathcal{D}| = 10 \) \textit{information sources or sensor platforms} deployed in the surveillance zone. These platforms can include radar systems, sonar arrays, infrared cameras, optical cameras, and acoustic sensors mounted on ships, submarines, or unmanned aerial vehicles (UAVs). Each platform (i.e., information source) has a likelihood function conditioned on the true hypothesis, which can be completely specified by the observationally equivalent sets $F_{\theta_p}(i)$ for all $\theta_p \in \Theta$ and for all $i \in \mathcal{D}$.  
    
\textbf{Observational Equivalence:}  
    Some sensors cannot distinguish between specific objects due to limitations in observation capabilities. For example, a sonar array may confuse a \textit{submarine} with a large \textit{cargo ship} due to similar acoustic signatures. As a result, for an information set, multiple classes can be observationally equivalent. The observationally equivalent sets \( F_{\theta_p}(i) \) for each \( \theta_p \in \Theta \) and \( i \in \mathcal{D} \) capture this ambiguity.

\begin{wrapfigure}{r}{0.5\textwidth}
\begin{center}
\includegraphics[height= 7cm]{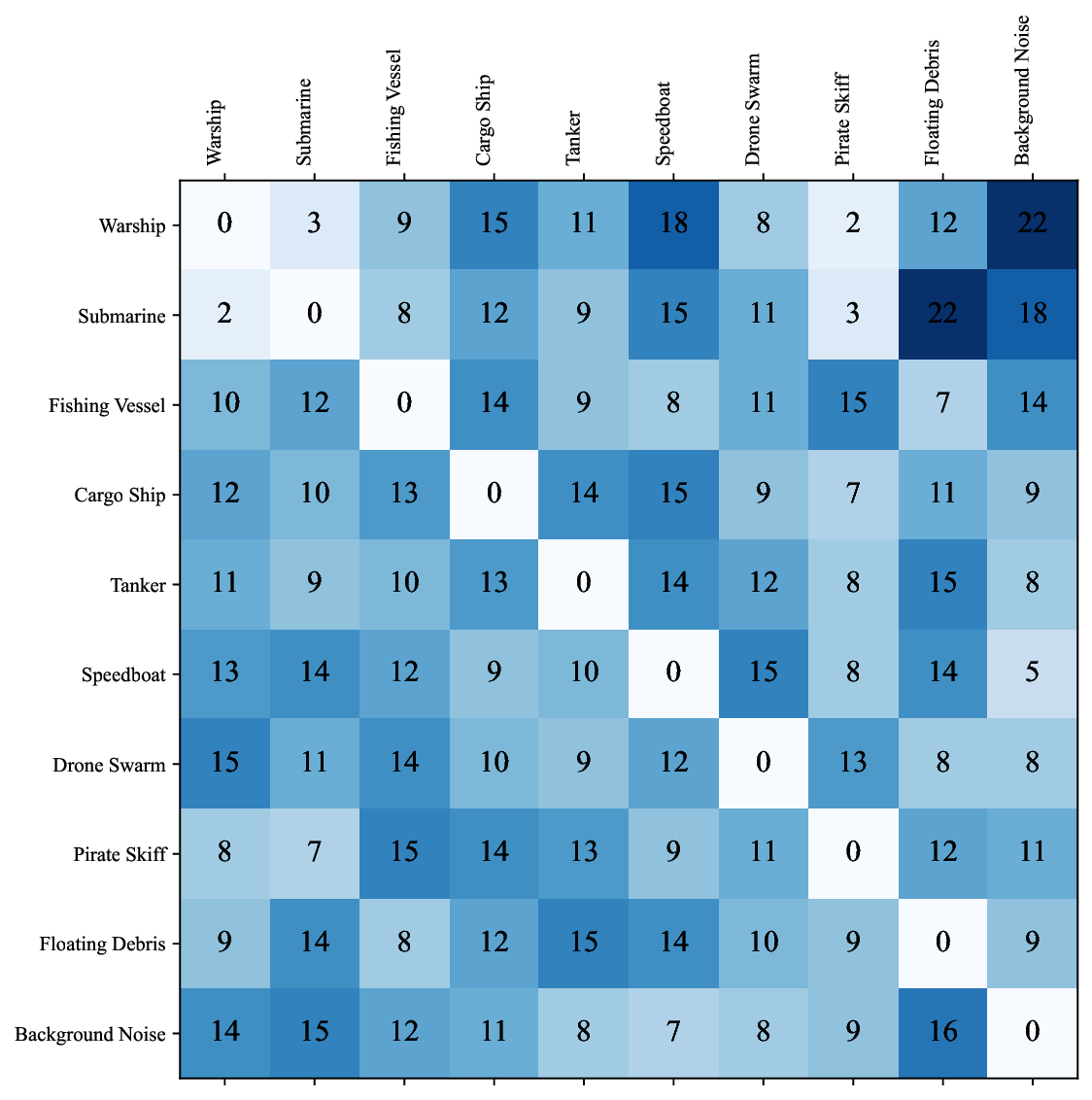}
\caption{Penalty Matrix for Maritime Threat Detection Task}
\label{fig1}      
\end{center}
\end{wrapfigure}

\begin{figure*}[!ht]
    \centering
    \begin{subfigure}{0.3\textwidth}
    \centering
        \includegraphics[width=160pt]{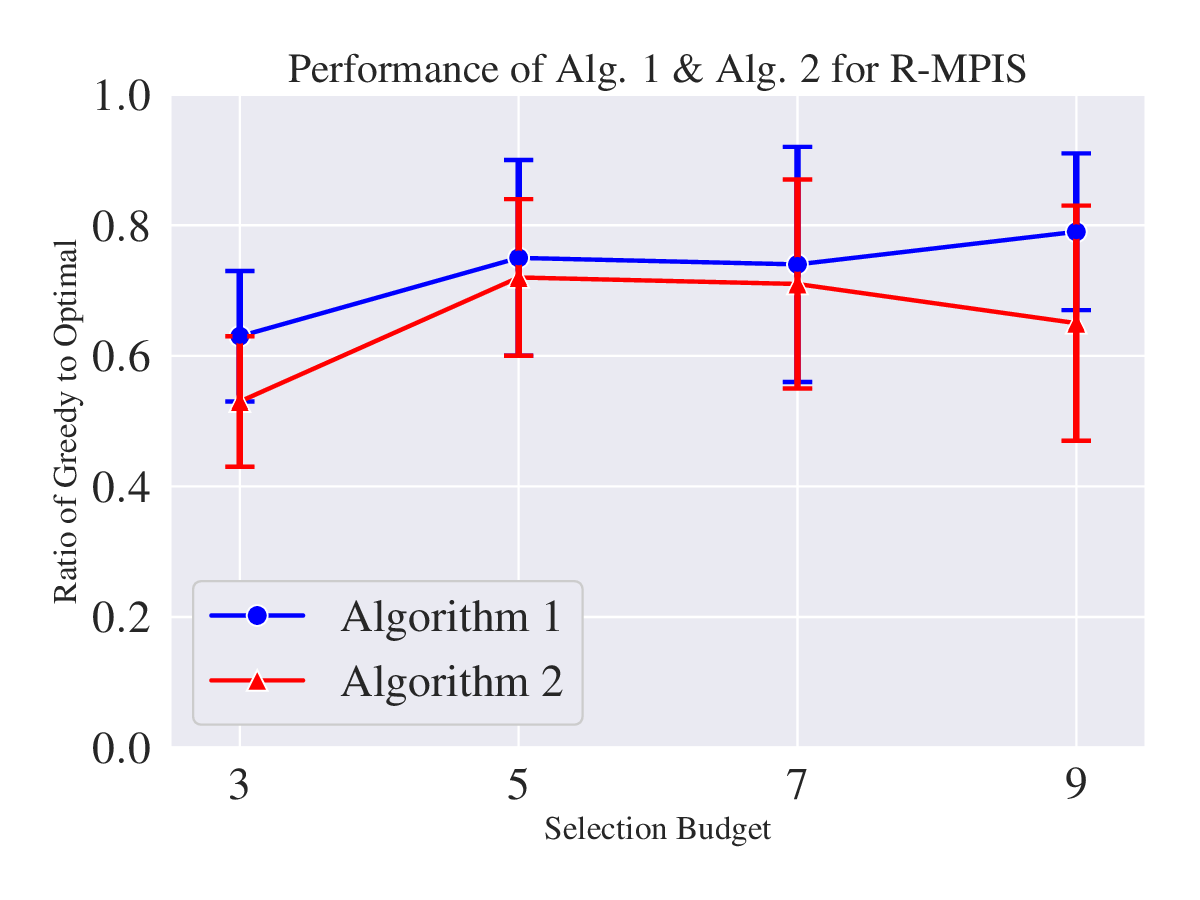}
        \caption{Ratio of greedy to optimal for R-MPIS with varying $K$ \& $A = \lceil 0.5 K \rceil$}
        \label{fig:a}
    \end{subfigure}
    \hfill
    \begin{subfigure}{0.3\textwidth}
    \centering
        \includegraphics[width=160pt]{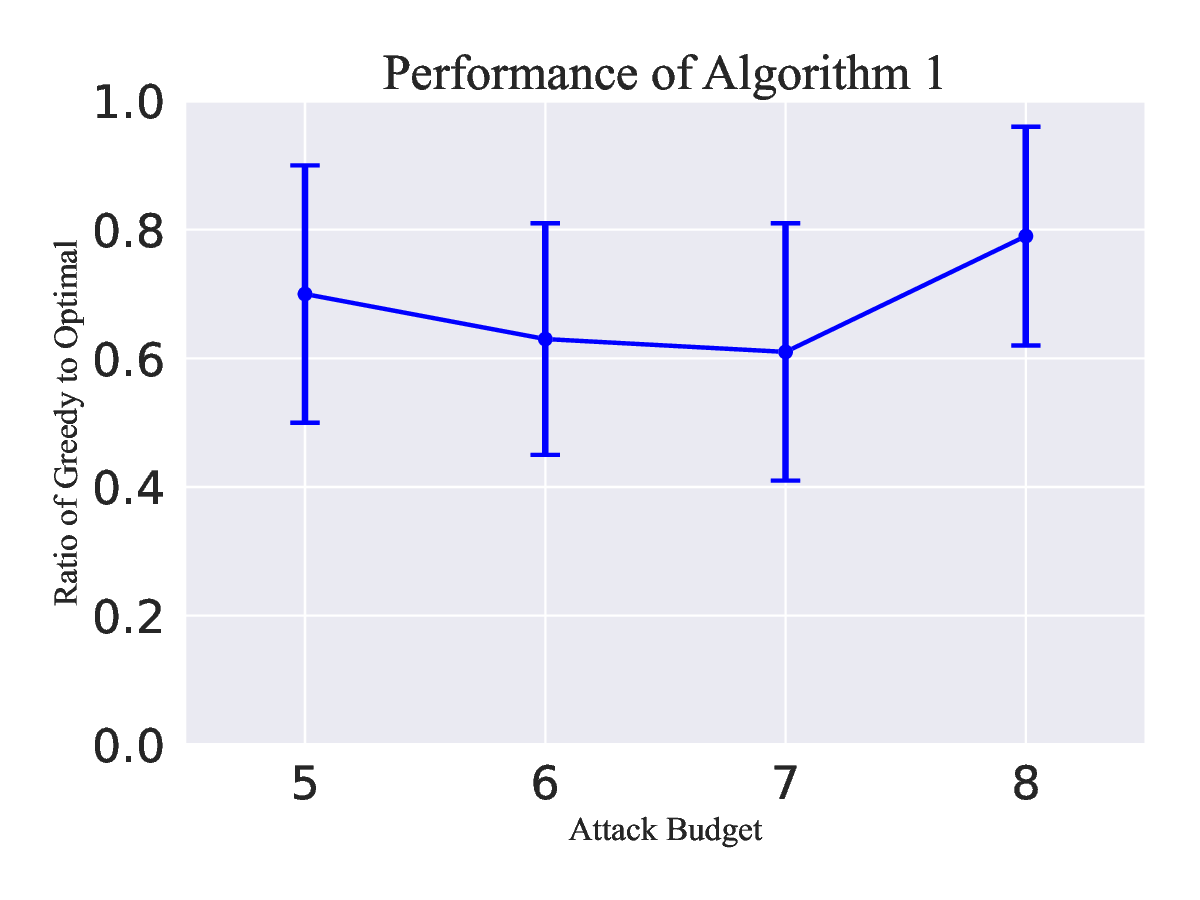}
        \caption{Ratio of greedy to optimal for R-MPIS with  $K = 9$ \& varying $A$ (Alg. \ref{algorithm:alg})}
        \label{fig:b}
    \end{subfigure}
    \hfill
    \begin{subfigure}{0.3\textwidth}
    \centering
        \includegraphics[width=160pt]{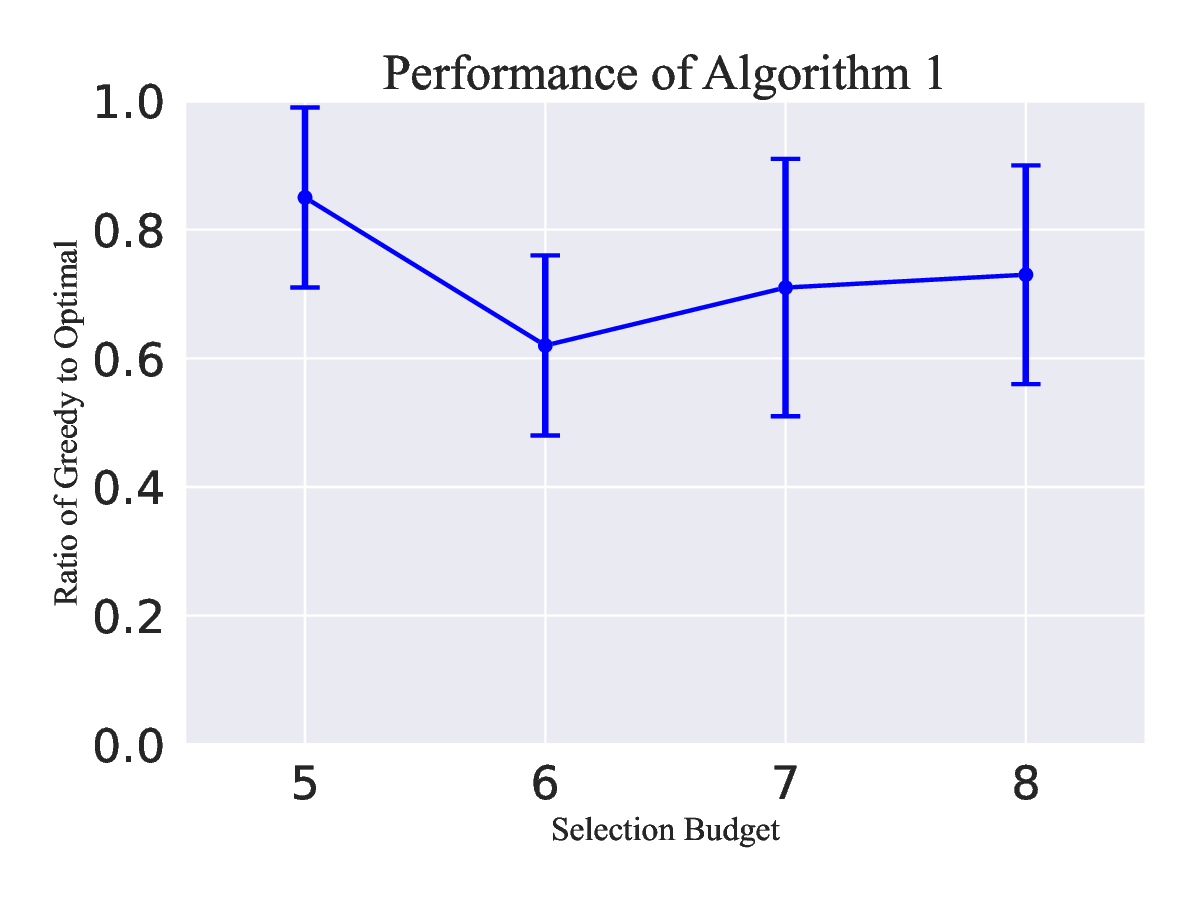}
        \caption{Ratio of greedy to optimal for R-MPIS with  $A = 3$ \& varying $K$ (Alg. \ref{algorithm:alg})}
        \label{fig:c}
    \end{subfigure}
    
    \medskip
    
 \begin{subfigure}{0.3\textwidth}
    \centering
        \includegraphics[width=160pt]{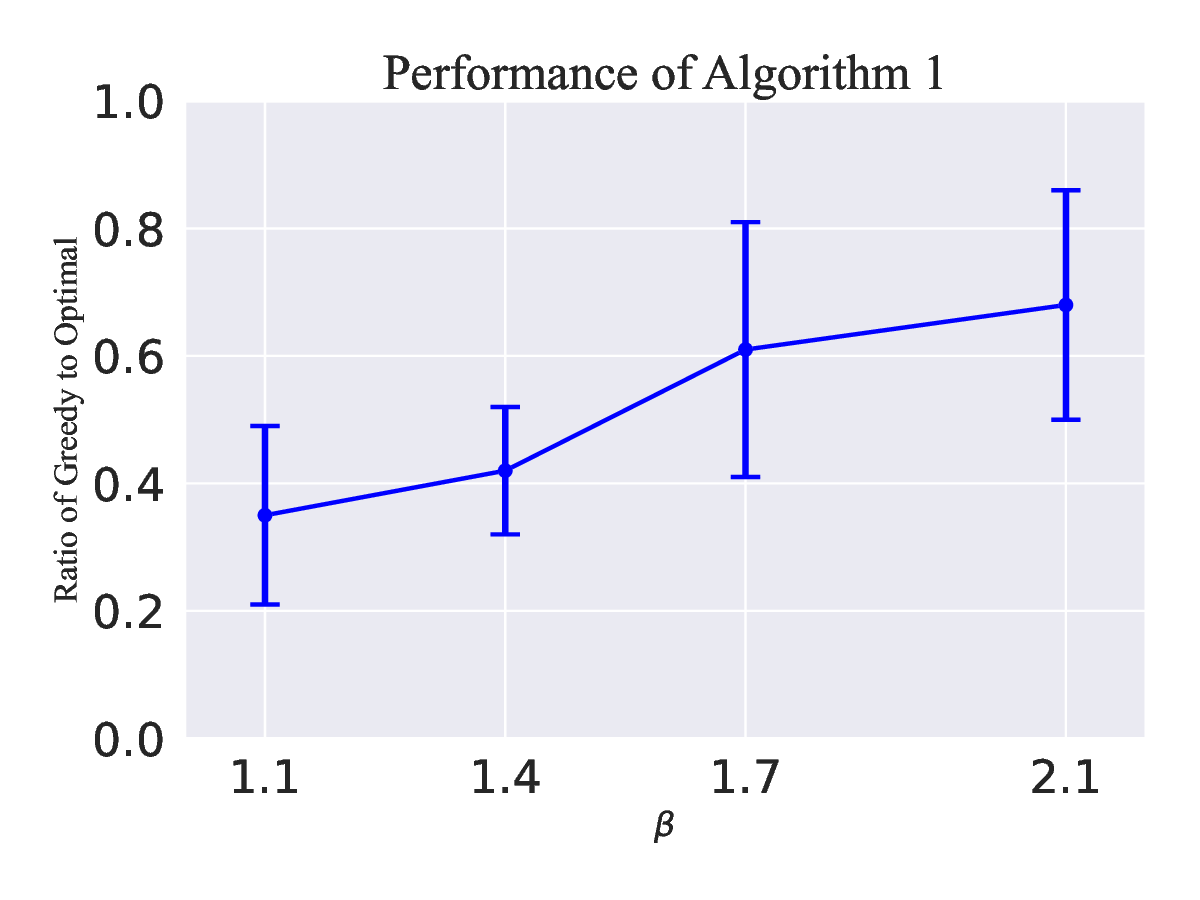}
        \caption{Ratio of greedy to optimal for R-MPIS with  $K = 9, A = 4$ \& varying $\beta$ (Alg. \ref{algorithm:alg})}
        \label{fig:d}
    \end{subfigure}
    \hfill
    \begin{subfigure}{0.3\textwidth}
    \centering
        \includegraphics[width=160pt]{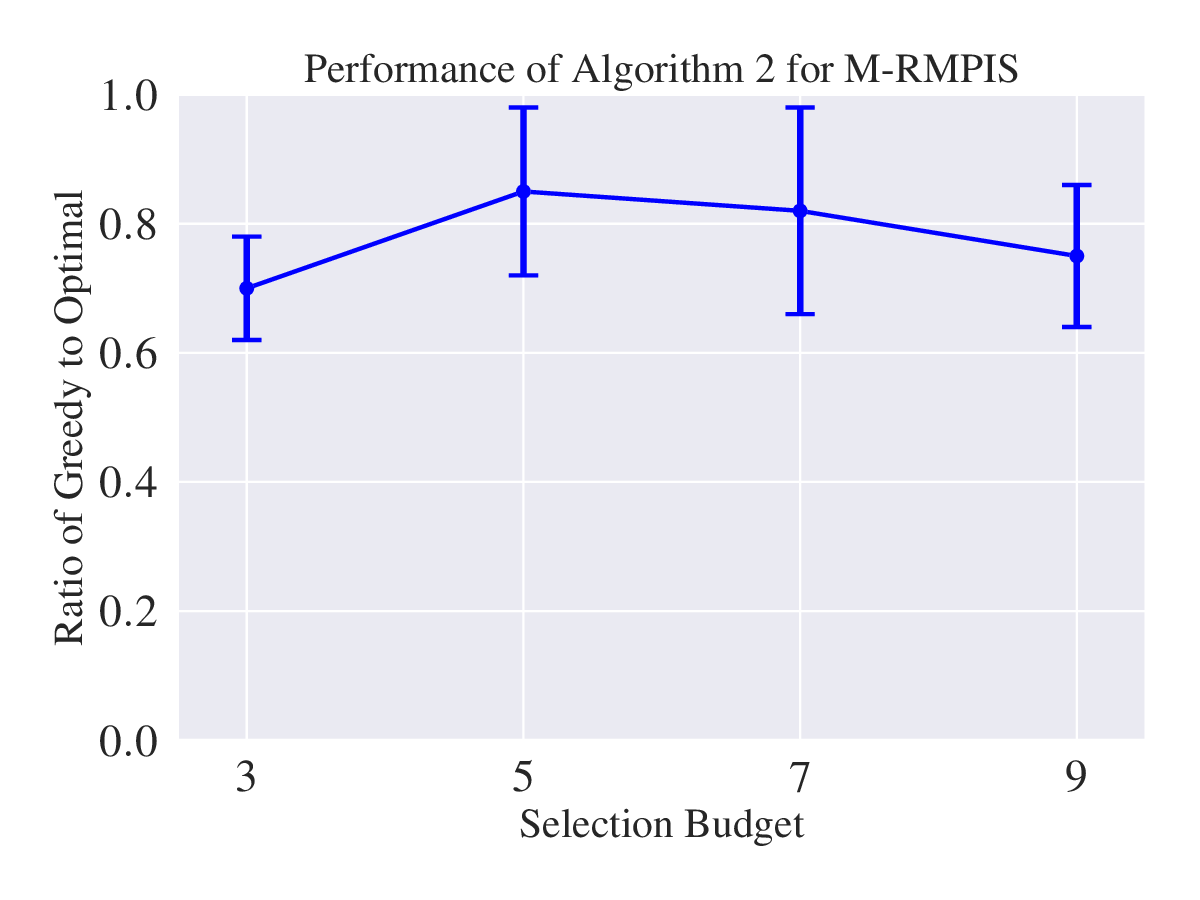}
        \caption{Ratio of greedy to optimal for M-RMPIS with varying $K$ \& $A = \lceil 0.5 K \rceil$}
        \label{fig:e}
    \end{subfigure}
    \hfill
    \begin{subfigure}{0.3\textwidth}
    \centering
        \includegraphics[width=150pt]{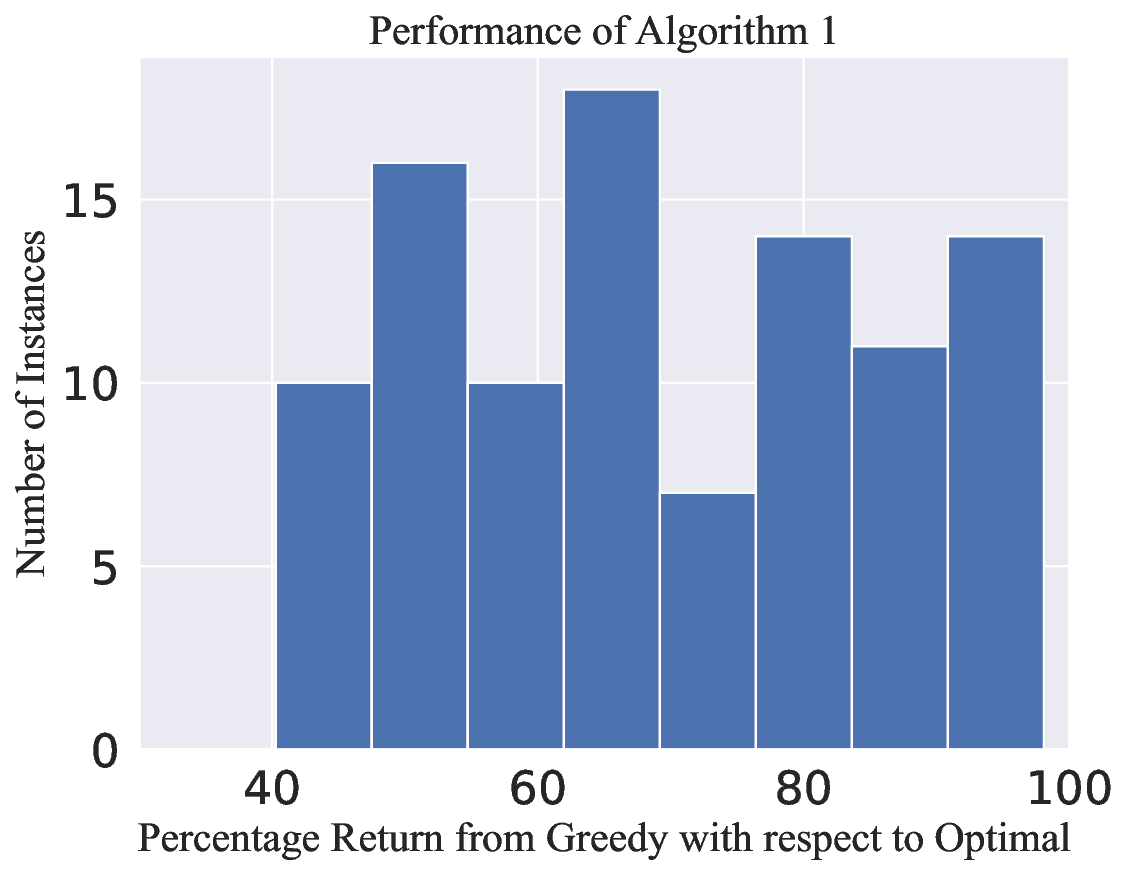}
        \caption{Evaluation of Alg. \ref{algorithm:alg} on randomly generated instances of R-MPIS}
        \label{fig:f}
    \end{subfigure}
    
    \medskip
    
        \begin{subfigure}{0.3\textwidth}
    \centering
     \includegraphics[width=150pt]{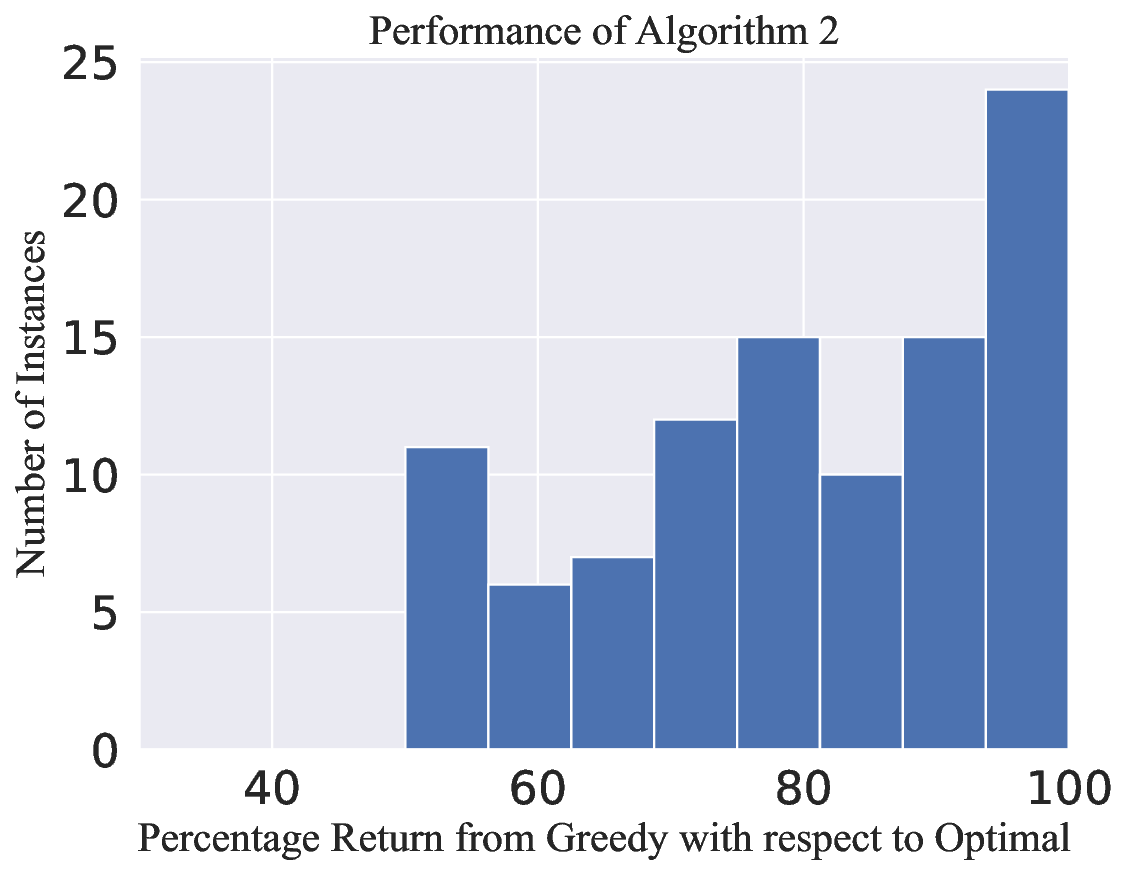}
        \caption{Evaluation of Alg. \ref{algorithm:alg2} on randomly generated instances of M-RMPIS}
        \label{fig:g}
    \end{subfigure}
    \hfill    
 \begin{subfigure}{0.3\textwidth}
    \centering
         \includegraphics[width=160pt]{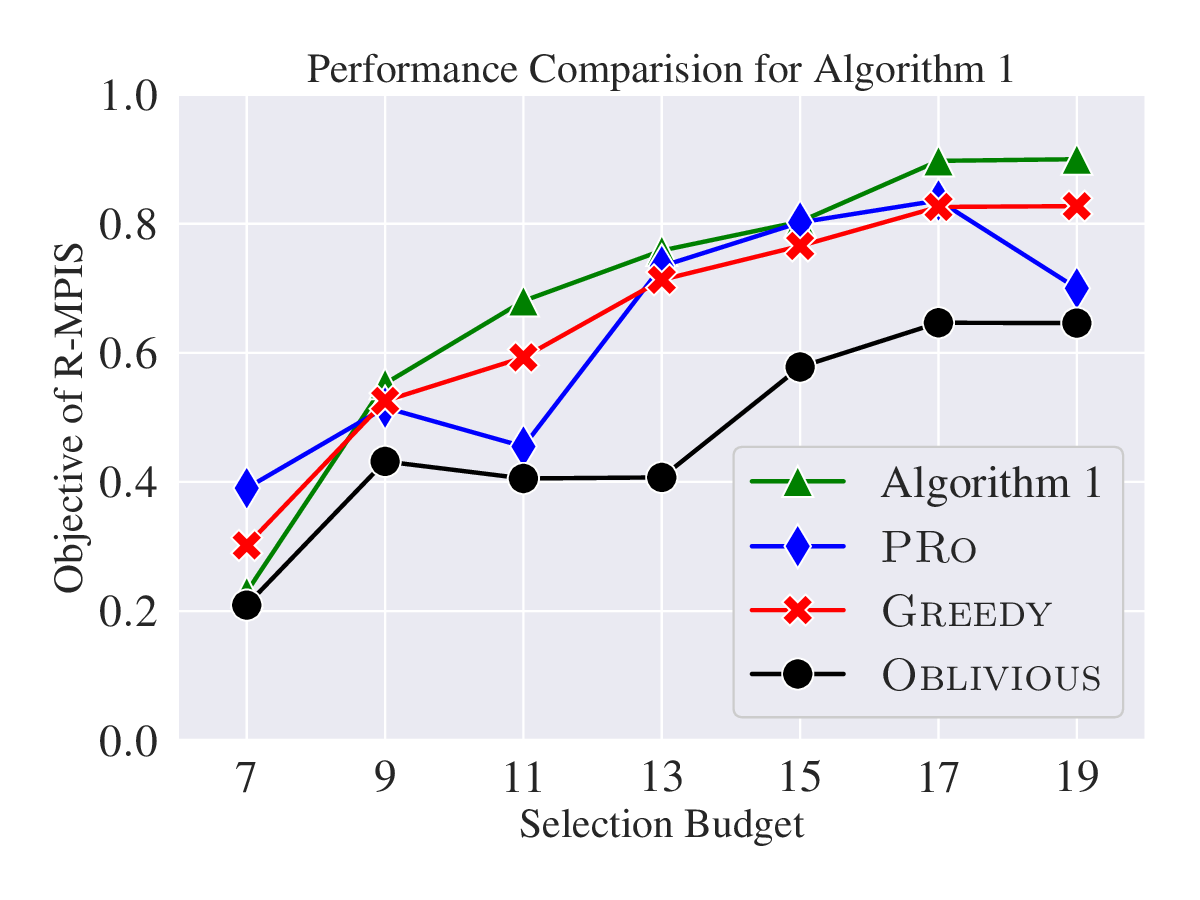}
        \caption{Performance comparison for Alg. \ref{algorithm:alg}}
        \label{fig:h}
    \end{subfigure}
    \hfill
    \begin{subfigure}{0.3\textwidth}
    \centering
        \includegraphics[width=160pt]{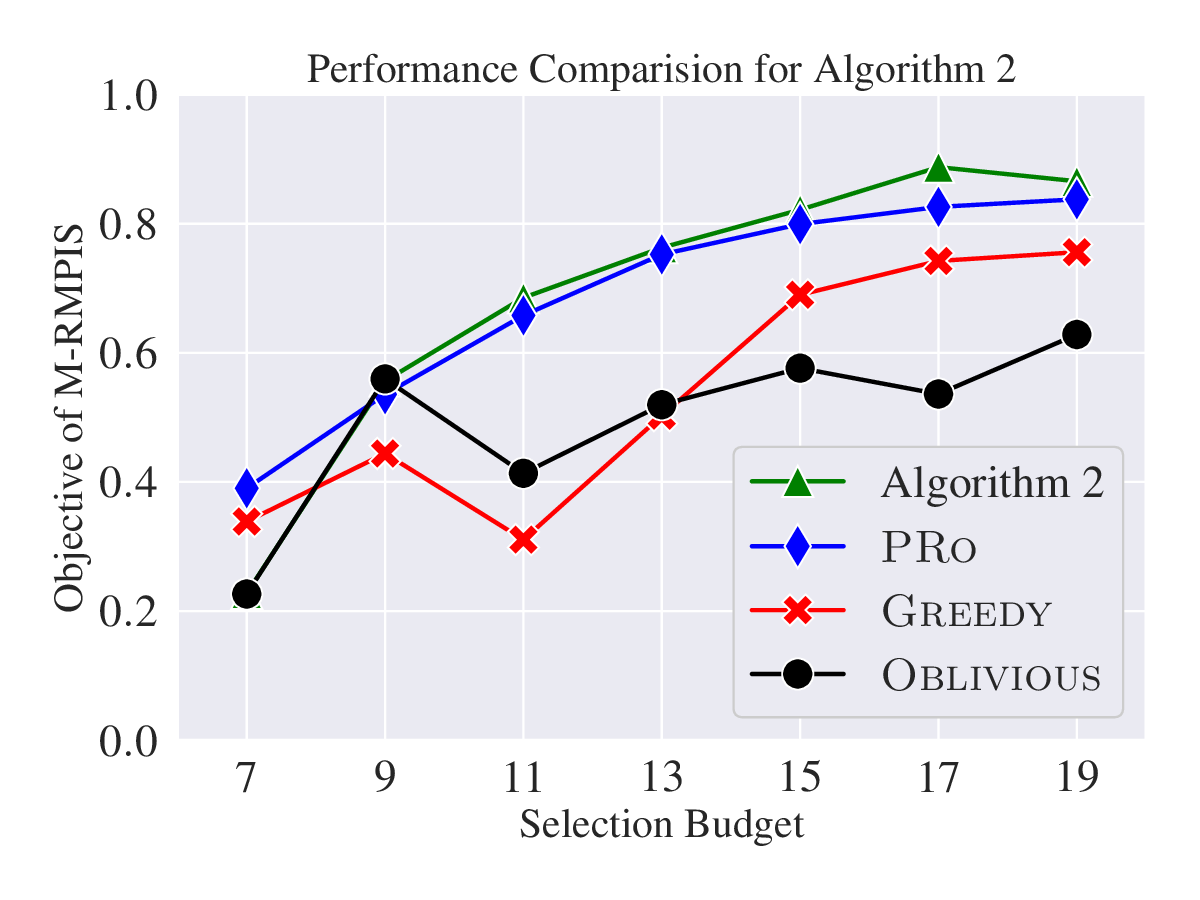}
        \caption{Performance comparison for Alg. \ref{algorithm:alg2}}
        \label{fig:i}
    \end{subfigure}
    \caption{Empirical evaluation of robust greedy algorithms (Alg. \ref{algorithm:alg} and Alg. \ref{algorithm:alg2}) on R-MPIS and M-RMPIS problem instances}
    \label{fig:all}
\end{figure*}
 For each instance of the R-MPIS problem, we randomly generate the observationally equivalent sets $F_{\theta_p}(i)$ for each $\theta_p \in \Theta$ and $i \in \mathcal{D}$. We perform the following evaluations:
 \begin{enumerate}
     \item We vary the selection budget $K \in \{3,5,7,9\}$ and set the respective attack budgets to $A = \lceil 0.5 K \rceil$.
     \item We fix the selection budget to $K=9$ and vary the attack budget $A \in \{ 5,6,7,8 \}$.
     \item We fix the attack budget to $A = 3$ and vary the selection budget $K \in \{5,6,7,8\}$.
     \item We fix the selection budget to $K=9$ and the attack budget to $A = 4$. We vary $\beta \in \{ 1.1, 1.4, 1.7, 2.1\}$.
 \end{enumerate}
 
 We run Algorithm \ref{algorithm:alg} to find the greedy information set $\mathcal{I}$ and find the optimal information set $\mathcal{I}^*$ using brute-force search. For both cases, we find the attack sets $\mathcal{I}'$ and $\mathcal{I}'^*$  by brute-force search. For each of the cases, we evaluate the algorithm for over 50 instances and plot the ratio of greedy utility to that of the optimal, i.e., $\Lambda(\mathcal{I} \setminus \mathcal{I}')/\Lambda(\mathcal{I^*}\setminus \mathcal{I}'^*)$, in Figures \ref{fig:a}, \ref{fig:b}, \ref{fig:c} and \ref{fig:d} respectively. We note that the misclassification penalties do not satisfy Assumption 2. As a result, for this instance, the performance bounds (of Algorithm \ref{algorithm:alg}) become trivial (i.e., $0$). Despite this, we observe near-optimal performance of the Algorithm \ref{algorithm:alg}. We also run Algorithm \ref{algorithm:alg2} using the surrogate objective $\Gamma(\cdot)$ to identify a greedy set $\mathcal{I}_g$, and find the respective (optimal) attack set $\mathcal{I}_g'$. We plot the ratio of $\Lambda(\mathcal{I}_g \setminus \mathcal{I}'_g)/\Lambda(\mathcal{I^*}\setminus \mathcal{I}'^*)$ in Figure \ref{fig:a}.  We observe that the information set obtained by optimizing for the surrogate objective provides near-optimal performance for the original objective (i.e., R-MPIS objective). Finally, we run Algorithm \ref{algorithm:alg2} for the given problem instance using the surrogate objective function (i.e., M-RMPIS problem) and plot the ratio of greedy to optimal utility $\Gamma(\mathcal{I} \setminus \mathcal{I}')/\Gamma(\mathcal{I}^* \setminus \mathcal{I}'^*)$ in Figure \ref{fig:e}. We observe near-optimal performance of Algorithm \ref{algorithm:alg2} which aligns with our theoretical bounds.

\subsection{Evaluation on Randomly Generated Instances}
We evaluate the performance of Algorithm \ref{algorithm:alg} and Algorithm \ref{algorithm:alg2} on several randomly generated instances of the R-MPIS and M-RMPIS problems, respectively. For each of the problems, we generate 100 random instances, where for each instance, we set the total number of data sources to $\mathcal{D}=10$, and generate a random penalty matrix $\Xi \in \mathbb{R}^{|\Theta| \times |\Theta|}$ with $|\Theta| = 10$. We randomly generate the observationally equivalent sets $F_{\theta_p}(i)$ for each $\theta_p \in \Theta$ and $i \in \mathcal{D}$. For each instance, we pick a selection budget $K$ uniformly at random from $\{5,...,10\}$ and the attack budget $A$ uniformly at random from $\{1,...,K-1\}$.

We run Algorithm \ref{algorithm:alg} for each instance to find the greedy set $\mathcal{I}$ and find the optimal attack set $\mathcal{I}'$ using a brute-force search. We find the optimal information set and attack sets $(\mathcal{I}^*, \mathcal{I}'^{*})$ through brute-force search. We plot the ratio of greedy to optimal utility $\Lambda(\mathcal{I}\setminus \mathcal{I}')/\Lambda(\mathcal{I}^*\setminus \mathcal{I}'^{*})$ in Figure \ref{fig:f}.  We note that for the R-MPIS problem, the randomly generated instances may not exhibit the weak submodularity property, since the misclassification penalties may not satisfy Assumption 2. As a result, for such instances, the performance bounds (of Algorithm \ref{algorithm:alg}) become trivial (i.e., $0$). Despite this, we observe near-optimal performance of the Algorithm \ref{algorithm:alg}, achieving an average of $76.44 \%$  of optimal utility. Similarly, we run Algorithm \ref{algorithm:alg2} for $100$ instances of the M-RMPIS problem and plot the ratio of greedy to optimal (modified) utility $\Gamma(\mathcal{I}\setminus \mathcal{I}')/\Gamma(\mathcal{I}^*\setminus \mathcal{I}'^{*})$ in Figure \ref{fig:g}. We observe a near-optimal performance of the Algorithm \ref{algorithm:alg2}, achieving an average of $80.61 \%$ of optimal utility. 

\subsection{Comparison with Baselines}
 We compare the utility provided by the information sets selected by Algorithm \ref{algorithm:alg} (for R-MPIS) and Algorithm \ref{algorithm:alg2} (for M-RMPIS) with the utility provided by the following algorithms: (i) Oblivious Selection (\textsc{Oblivious}) (Baseline considered in \cite{bogunovic2018robust}), (ii) Vanilla-Greedy Algorithm (\textsc{Greedy}) (Algorithm 1 in \cite{bhargav2023complexity}) and (iii) Partitioned Robust Submodular Optimization Algorithm (\textsc{PRo}) (Algorithm 1 in \cite{bogunovic2017robust}). We create an instance of the R-MPIS Problem with $\Theta=10$ by randomly generating a penalty matrix. We set the number of information sources to $\mathcal{D} = 20$ and randomly generate the observationally equivalent sets $F_{\theta_p}(i)$ for each $\theta_p \in \Theta$ and $i \in \mathcal{D}$. We set the attack budget to $A = 5$ and vary the selection budget $K \in \{7,9,11,13,15,17, 19\}$. For every instance of the R-MPIS problem, we run Algorithm \ref{algorithm:alg} with $\beta = \lfloor K/A \rfloor$. We compare the utility provided by the information sources (which remain after an attacker has best responded with an attack set) selected by the respective algorithms in Figures \ref{fig:h} and \ref{fig:i}. For R-MPIS, we observe in Figure \ref{fig:h} that Algorithm \ref{algorithm:alg} outperforms all baselines. Furthermore, we also observe that the utility increases with an increase in the selection budget. For M-RMPIS, we observe in Figure \ref{fig:i} that Algorithm \ref{algorithm:alg2} and \textsc{PRo} show similar performance, since the objective is submodular. However, Algorithm \ref{algorithm:alg2} is more efficient and scalable since \textsc{PRo} constructs multiple sets (or buckets) whose number and size depend on the selection and attack budgets $(K,A)$, while Algorithm \ref{algorithm:alg2} always constructs two sets $\mathcal{A}_1$ and $\mathcal{A}_2$.

\section{Conclusion}
In this work, we studied the robust information selection problem for a hypothesis testing/ classification task under adversarial attacks or deletions of the selected information sources. We proposed a novel misclassification penalty framework which enables non-uniform treatment of classification errors, and captures the varying importance of misclassification events. We characterized the weak-submodularity and curvature properties of the objective function and proposed an efficient greedy algorithm with provable near-optimality guarantees for minimizing the maximum misclassification penalty under failures or adversarial deletions. Recognizing the limitations of the maximum penalty metric, we introduced a submodular surrogate based on the total penalty of misclassification and proposed a greedy algorithm with stronger performance guarantees. Finally, we empirically demonstrated the effectiveness of our algorithms in various instances, showing their robustness and near-optimal performance. 

\bibliographystyle{ieeetr} 
\bibliography{main}

\end{document}